\documentclass{article}

\usepackage{microtype}
\usepackage{graphicx}
\usepackage{booktabs} %
\usepackage{paralist, tabularx}
\usepackage{multirow}
\usepackage{subcaption}
\usepackage{hyperref}

\usepackage[dvipsnames,table,xcdraw]{xcolor}
\usepackage{amsmath}
\usepackage{amssymb}
\usepackage{mathtools}
\usepackage{amsthm}
\usepackage{placeins}
\usepackage[accepted]{icml2024}

\usepackage{amsmath,amssymb,amsfonts,bm,mathtools}
\usepackage{amsthm}
\usepackage{graphicx}

\newcommand{\dx}{\mathrm{d}\vx}

\newcommand{\veta}{\bm{\eta}}

\renewcommand{\d}{\mathrm{d}}
\theoremstyle{plain}
\newtheorem{theorem}{Theorem}[section]

\newtheorem{lemma}[theorem]{Lemma}

\theoremstyle{definition}
\newtheorem{definition}[theorem]{Definition}

\theoremstyle{remark}

\usepackage{wasysym}
\usepackage{mathtools}
\usepackage{authblk}
\usepackage{algorithmic}
\usepackage{subcaption}
\usepackage{wrapfig}
\usepackage{resizegather}

\def\eqref#1{equation~\ref{#1}}
\def\Eqref#1{Equation~\ref{#1}}

\numberwithin{equation}{section}

\def\1{\bm{1}}

\def\vg{{\bm{g}}}
\def\vh{{\bm{h}}}

\def\vw{{\bm{w}}}
\def\vx{{\bm{x}}}

\def\vz{{\bm{z}}}

\def\vX{{\bm{X}}}
\def\vY{{\bm{Y}}}
\def\vZ{{\bm{Z}}}
\def\v0{\bm{0}}

\DeclareMathAlphabet{\mathsfit}{\encodingdefault}{\sfdefault}{m}{sl}
\SetMathAlphabet{\mathsfit}{bold}{\encodingdefault}{\sfdefault}{bx}{n}

\newcommand{\E}{\mathbb{E}}

\usepackage{tikz}
\usetikzlibrary{backgrounds, arrows.meta, shapes, tikzmark, calc}
\usepackage{annotate-equations}

\icmltitlerunning{Consistent Diffusion Meets Tweedie}
\begin{document}

\twocolumn[\icmltitle{Consistent Diffusion Meets Tweedie: \\ Training Exact Ambient Diffusion Models with Noisy Data}

\icmlsetsymbol{equal}{*}

\begin{icmlauthorlist}
    \icmlauthor{Giannis Daras}{ut_austin_cs,archimedes}
    \icmlauthor{Alexandros G. Dimakis}{ut_austin_ece}
    \icmlauthor{Constantinos Daskalakis}{mit,archimedes}
\end{icmlauthorlist}

\icmlaffiliation{ut_austin_cs}{Department of Computer Science, University of Texas at Austin}
\icmlaffiliation{ut_austin_ece}{Department of Electrical and Computer Engineering, University of Texas at Austin}
\icmlaffiliation{mit}{Department of Electrical Engineering and Computer Science, MIT}
\icmlaffiliation{archimedes}{Archimedes AI}

\icmlcorrespondingauthor{Giannis Daras}{giannisdaras@utexas.edu}
\icmlcorrespondingauthor{Alexandros G. Dimakis}{dimakis@austin.utexas.edu}
\icmlcorrespondingauthor{Constantinos Daskalakis}{costis@csail.mit.edu}

\icmlkeywords{Machine Learning, ICML, diffusion models, ambient diffusion, corrupted data, tweedie}

\vskip 0.3in
]

\printAffiliationsAndNotice{\icmlEqualContribution} %

\newcommand{\ap}[1]{\textcolor{black}{#1}}
\newcommand{\cnote}[1]{\textcolor{blue}{#1}}

\begin{abstract}
    Ambient diffusion is a recently proposed framework for training diffusion models using corrupted data. Both Ambient Diffusion and alternative SURE-based approaches for learning diffusion models from corrupted data resort to approximations which deteriorate performance. 
    We present the first framework for training diffusion models that provably sample from the uncorrupted distribution given only noisy training data, solving an open problem in Ambient diffusion. Our key technical contribution is a method that uses a double application of Tweedie's formula and a consistency loss function that allows us to extend sampling at noise levels below the observed data noise. 
    We also provide further evidence that diffusion models memorize from their training sets by identifying extremely corrupted images that are almost perfectly reconstructed, raising copyright and privacy concerns.  
    Our method for training using corrupted samples can be used to mitigate this problem. We demonstrate this by fine-tuning Stable Diffusion XL to generate samples from a distribution using only noisy samples. Our framework reduces the amount of memorization of the fine-tuning dataset, while maintaining competitive performance.
\end{abstract}

\begin{figure*}[!ht]
    \centering
    \includegraphics[width=0.8\textwidth]{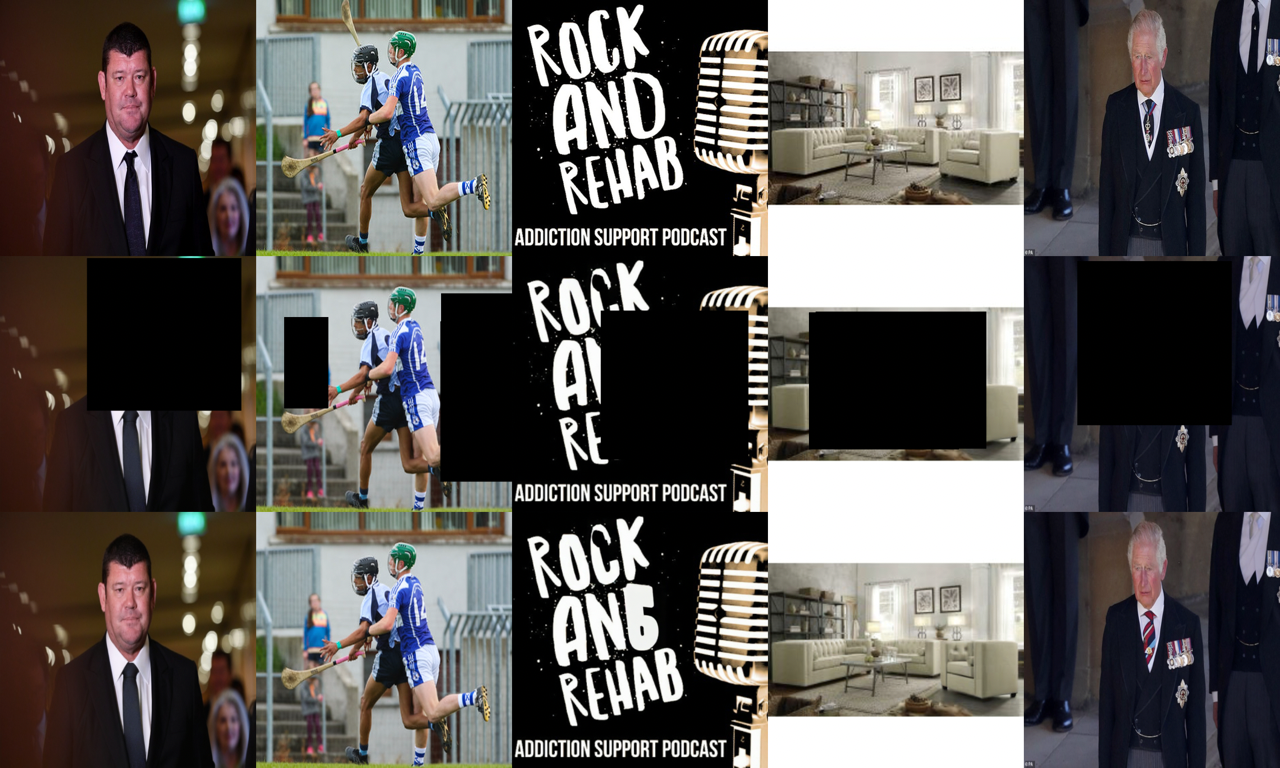}
    \caption{Top row: images from LAION~\citep{schuhmann2022laion}, middle row: masked images, bottom row: reconstructed images with the SDXL~\citep{podell2023sdxl} inpainting model. The accuracy of the reconstructions presents strong evidence that the images on the top-row were in the training set of SDXL (or SDXL Inpainting) and have been memorized. To the best of our knowledge, SDXL does not disclose its training set.}
    \label{fig:sdxl_inp_attack}
\end{figure*}

\begin{figure}[!ht]
  \centering
  \begin{subfigure}{0.5\textwidth}
    \includegraphics[width=\linewidth]{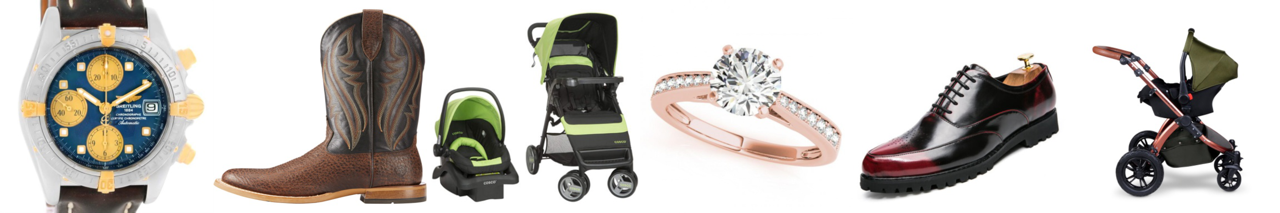}
    \caption{Images from LAION.}
  \end{subfigure}
  \begin{subfigure}{0.5\textwidth}
    \includegraphics[width=\linewidth]{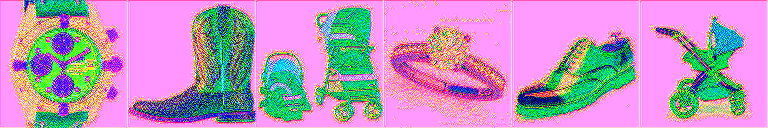}
    \caption{Corresponding latents.}
  \end{subfigure}
  \begin{subfigure}{0.5\textwidth}
    \includegraphics[width=\linewidth]{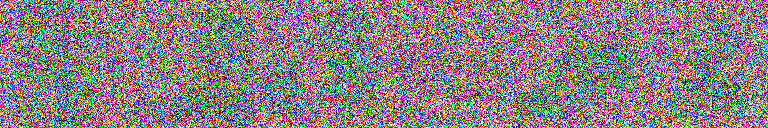}
    \caption{Noisy Latents at level $t=900$.}
  \end{subfigure}
  \begin{subfigure}{0.5\textwidth}
    \includegraphics[width=\linewidth]{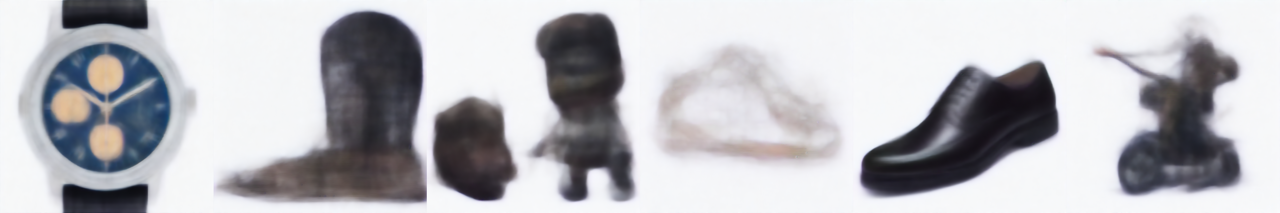}
    \caption{MMSE.}
  \end{subfigure}
  \begin{subfigure}{0.5\textwidth}
    \includegraphics[width=\linewidth]{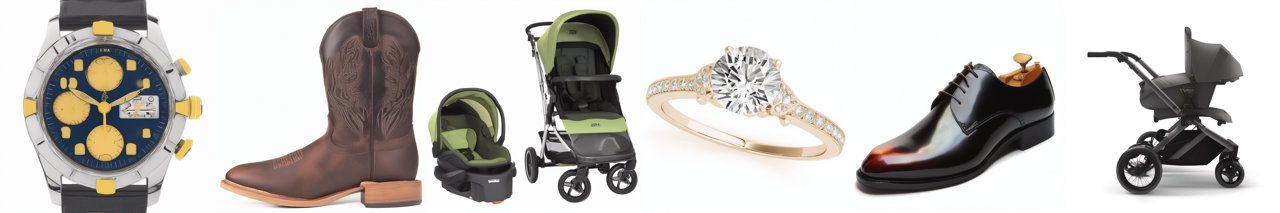}
    \caption{Generated images.}
  \end{subfigure}
  \caption{SDXL~\citep{podell2023sdxl} posterior samples (Row e) given extremely noisy encodings (Row c) of LAION images (Row a). The level of fidelity of the reconstructions to the original images, despite the severe corruption (c) and the blurriness of the MMSE solution (d), indicates that the images were potentially in the training set and have been memorized.}
  \label{fig:sdxl_noise_attack}
\end{figure}

\section{Introduction}
\label{sec:introduction}
In recent years, we have witnessed remarkable progress in image generation as exemplified by state-of-the-art models such as Stable Diffusion (-XL)~\citep{ldm, podell2023sdxl} and DALL-E (2, 3)~\citep{dalle3}. This progress has been driven by two major enablers: i)~the diffusion modeling framework~\citep{ddpm, ncsn, ncsnv3}; and ii)~the existence of massive datasets of image-text pairs~\citep{schuhmann2022laion, gadre2023datacomp}.\looseness=-1 

The need for high-quality, web-scale data and the intricacies involved in curating datasets at that scale often result in the inclusion of copyrighted content. Making things worse, diffusion models memorize training examples more than previous generative modeling approaches~\citep{carlini2023extracting, somepalli2023understanding}, such as Generative Adversarial Networks~\citep{goodfellow2020generative}, often replicating parts or whole images from their training set.

A recently proposed  strategy for mitigating the memorization issue is to train (or fine-tune) diffusion models using corrupted data~\cite{daras2023ambient, somepalli2023understanding, daras2023solving}. Indeed, developing a capability for training diffusion models using corrupted data can also find applications in domains where access to uncorrupted data is expensive or impossible, e.g.~in MRI~\citep{aali2023solving} or black-hole imaging~\citep{linimaging, Akiyama_2019}. Unfortunately, existing methods for learning diffusion models from corrupted data~\citep{daras2023ambient, aali2023solving, kawar2023gsure, xiang2023ddm} resort to approximations (during training or sampling) that significantly hurt performance. Our contributions are as follows:
\begin{enumerate}[i]
    \item We propose the first \textit{exact} framework for learning diffusion models using only corrupted samples. Our key technical contributions are: i) a computationally efficient method for learning  optimal denoisers for all levels of noise $\sigma \geq \sigma_n$, where $\sigma_n$ is the standard deviation of the noise in the training data, obtained by applying Tweedie's formula twice; and ii) a consistency loss function~\citep{daras2023consistent} for learning the optimal denoisers for noise levels $\sigma \leq \sigma_n$. Note that given samples at level of noise $\sigma_n$ it is possible to obtain samples at  levels of noise $\sigma > \sigma_n$ (by adding further noise) but prior to our work it was not known how to train diffusion models to obtain samples at levels of noise $\sigma < \sigma_n$.
    
    \item We provide further evidence that foundation diffusion models memorize their training sets by showing that extremely corrupted training images can be almost perfectly reconstructed. Moreover, we show that memorization occurs at a higher rate than previously anticipated.\looseness=-1

    \item We use our framework to fine-tune diffusion foundation models using corrupted data and show that the performance of our trained model declines (as the corruption in the training data increases) at a much slower rate compared to previously proposed approaches.
    
    \item We evaluate trained models against our as well as a baseline method for testing data replication and we show that models trained under data corruption memorize significantly less.

    \item We open-source our code to facilitate further research in this area: \href{https://github.com/giannisdaras/ambient-tweedie}{https://github.com/giannisdaras/ambient-tweedie}.
\end{enumerate}

\section{Background and Related Work}
\label{sec:background}

Consider a distribution of interest  admitting a density function $p_0$. Our goal is to train a diffusion model that generates samples from $p_0$. However, we only have access to noisy samples from $p_0$. In particular, we have samples of the form $\vX_{t_n} = \vX_0 + \sigma_{t_n} \vZ$, where $\vX_0 \sim p_0$ and  $\vZ \sim \mathcal N(\v0, I_d)$. We denote by $p_{t_n}$ the distribution density of these samples. Throughout the paper we fix an increasing non-negative function $\sigma(t)$, where $t \in [0,T]$, $T>0$, and $\sigma(0)=0$, and denote $\sigma(t)$ by $\sigma_t$. We take $t_n \in (0,T)$. The subscript `$n$' in $t_n$ refers to ``nature'' and, as stated above, we assume that nature is giving us access to samples at noise level $\sigma_{t_n}$. We denote by $p_t$ the distribution of random variable $\vX_{t} = \vX_0 + \sigma_{t} \vZ$, where $\vX_0 \sim p_0$ and  $\vZ \sim \mathcal N(\v0, I_d)$.

\subsection{Background on denoising diffusion models} Diffusion models can equivalently be viewed as denoisers at many different noise levels $\sigma_t$, $t \in [0,T]$. They are typically trained with the Denoising Score Matching loss:
\vspace{-0.25em}
\begin{gather*}
    J_{\mathrm{DSM}}(\theta) = \nonumber \\ \E_{\vx_0 \sim p_0(\vx_0)} \E_{t \sim \mathcal U[0, T]} \E_{\vx_t \sim p_t(\vx_t|\vx_0)}\left[ \left| \left| \vh_{\theta}(\vx_t, t) - \vx_0 \right|\right|^2\right].
    \label{eq:dsm_loss}
\end{gather*}
If the function class $\{\vh_{\theta}\}$ is sufficiently rich, the minimizer of this loss satisfies $\vh_{\theta^*}(\vx_t, t) = \E[\vX_0 | \vX_t = \vx_t]$ for all $t, \vx_t$. Tweedie's formula connects the conditional expectation, i.e. the best denoiser in the $\ell_2^2$ sense, with the score function $\nabla \log p_t(\vx_t)$,
\begin{gather}
    \nabla \log p_t(\vx_t) = \frac{\E[\vX_0 | \vX_t = \vx_t] - \vx_t}{\sigma_t^2}.
    \label{eq:tweedies_main_paper}
\end{gather}
The score can be then used to sample from $p_0(\vx_0)$~\citep{anderson}, by sampling a trajectory from the following Stochastic Differential Equation (which is the reverse diffusion process of the process that adds Gaussian noise to a sample from $p_0$ according to the noise schedule $\sigma_t$):
\begin{gather}
    \dx_t = -2\sigma_t \d \sigma_t\nabla \log p_t(\vx_t) + \sqrt{\frac{\d \sigma_t^2}{\d t}} \d\bar\vw,
\end{gather}
initialized at $\vx_T \sim p_T(\vx_T)$, where $\bar \vw$ is a standard Wiener process when the time flows backwards (from $T$ to 0). In practice, the score function is replaced with its estimate, i.e. we run the process:
\begin{gather}
    \dx_t = -2\d \sigma_t\frac{\vh_{\theta}(\vx_t, t) - \vx_t}{\sigma_t} + \sqrt{\frac{\d \sigma_t^2}{\d t}} \d\bar\vw.
    \label{eq:reverse_theta}
\end{gather}

As we have described, in our setting of interest, we do not have access to samples from $p_0$, thus we are not in a position to train a diffusion model using~\Eqref{eq:dsm_loss}. Yet, we still want to learn $\E[\vX_0 | \vX_t = \vx_t]$ (and thus, the score) for \textit{all noise levels $t$}. A natural first question is whether we can at least learn denoisers for noise levels that are equal to that of the available data or larger, i.e.~for $t: \sigma_t \geq \sigma_{t_n}$.

\subsection{Prior Work on learning denoisers from noisy data} Prior work has given an affirmative answer to the question at the end of the previous section. One of the most established methods is \textbf{S}tein’s \textbf{U}nbiased \textbf{R}isk \textbf{E}stimate (\textbf{SURE})~\citep{stein1981estimation}. SURE learns the conditional expectation of $\vX_0$ given a sample $\vX_{t_n}$, by minimizing the following objective \textit{that only uses the noisy realizations}:
\begin{gather}
    J_{\mathrm{SURE}}(\theta) = \nonumber \\  \E_{\vx_{t_n} \sim p_{t_n}}\left[\left|\left| \vh_{\theta}(\vx_{t_n}) -\vx_{t_n}\right|\right|^2 + 2\sigma_{t_n}^2 (\nabla_{\vx} \cdot \vh_{\theta}(\vx_{t_n}))^2\right], \nonumber
    \label{eq:sure_main_paper}
\end{gather}
The divergence term is expensive to evaluate and is typically replaced with the Monte Carlo approximation:
\begin{gather*}
    (\nabla_{\vx} \cdot \vh_{\theta}(\vx_{t_n}))^2 \approx \vz^T \left( \frac{\vh_{\theta}(\vx_{t_n} + \epsilon\vz) - \vh_{\theta}(\vx_{t_n}) }{\epsilon} \right),
\end{gather*}
for some small, positive parameter $\epsilon$ and $\vz \sim N(\v0, I_d)$~\citep{aggarwal2022ensure, soltanayev2018training, metzler2018unsupervised}. An alternative approximation is to compute the Jacobian Vector Product $\vz^T \nabla \vh_{\theta}(\vx_{t_n})\vz, \ \vz \sim \mathcal N(\v0, I_d)$, with automatic differentiation tools~\citep{kawar2023gsure}. Both methods give unbiased estimators for the divergence term and the variance can be decreased by averaging over many $\vz$ (at the cost of increased computation).

Another line of work is the Noise2Noise~\citep{noise2noise} framework and its generalizations~\citep{batson2019noise2self, krull2019noise2void, pang2021recorrupted, xu2020noisy, moran2020noisier2noise}. Most relevant to our work are the Noisier2Noise~\citep{moran2020noisier2noise} and Noisy-As-Clean~\citep{xu2020noisy} approaches wherein the training goal is to predict the noisy signal from a further corrupted version of it. Noisier2Noise comes with the theoretical guarantee of learning  $\E[\vX_0 | \vX_t = \vx_t]$ \textit{for a single} $t: \sigma_t \geq \sigma_{t_n}$, where by ``a single $t$'' we mean that, for each $t$ of interest, a new problem needs to be solved.

\subsection{Prior work on learning diffusion models from corrupted data}
\label{subsec:prior_work_ambient_diffusion}
In contrast to the previous section, there is no known approach for learning $\E[\vX_0 | \vX_t = \vx_t]$ for noise levels $t: \sigma_t \leq \sigma_{t_n}$. Thus, it is not known how to train an exact diffusion model using noisy samples, so various approximations have been considered, as described below.

\citet{aali2023solving} uses SURE to learn the optimal denoiser at the noise level of the available data, i.e. $\E[\vX_0 | \vX_{t_n}=\vx_{t_n}]$, and then creates iterates $\tilde \vX_t = \E[\vX_0 | \vX_{t_n}] + \sigma_{t}\vZ$ at all noise levels $t$ to train with Denoising Score Matching. However, the underlying noisy distributions, $\tilde p_t(\vx_t)$, are different than $p_t(\vx_t)$, since the R.V. $\vX_0$ has been replaced with $\E[\vX_0 | \vX_{t_n}]$. An alternative approach is to use SURE to learn $\E[\vX_0 | \vX_t]$ for all levels $t: \sigma_t \geq \sigma_{t_n}$ and early stop diffusion sampling at time $t: \sigma_t = \sigma_{t_n}$. This type of approach is adopted by \citet{kawar2023gsure} and \citet{xiang2023ddm} and it only guarantees samples from the distribution $\E[\vX_0 | \vX_{t_n}]$.

Notably, similar problems arise in the setting of training diffusion models from linearly corrupted data, i.e. when the available samples are $\vY_0 = A\vX_0$, for a known matrix $A$, as considered in the Ambient Diffusion paper~\citep{daras2023ambient}. In this setting, the authors manage to learn $\E[\vX_0 | A\vX_t]$ for all $t$, but not $\E[\vX_0 | \vX_t]$, where $\vX_t = \vX_0 + \sigma_t Z$, as always in this paper. Similar challenges are encountered in the G-SURE paper~\citep{kawar2023gsure}.

In sum, training exact diffusion models, i.e.~diffusion models sampling the target distribution $p_0$, given corrupted data remains unsolved. In this paper, we resolve this open-problem with two key technical contributions:  i) an  efficiently computable objective for learning the optimal denoisers for all levels of noise  $t:\sigma_t \geq \sigma_{t_n}$, obtained by applying Tweedie's formula twice; and ii) a consistency loss for learning the optimal denoisers for levels of noise  $t:\sigma_t \leq \sigma_{t_n}$. We describe these contributions in the next section.

\section{Method}
\label{sec:method}

\subsection{Learning the Optimal Denoiser for $\sigma_t > \sigma_{t_n}$} \label{subsec:double_tweedie}We first present an  efficiently computable objective that resembles Denoising Score Matching and enables learning the optimal denoisers for all noise levels $t: \sigma_t > \sigma_{t_n}$.

\begin{theorem}[Ambient Denoising Score Matching]
Define $\vX_t$ as in the beginning of Section~\ref{sec:background}. Suppose we are given samples $\vX_{t_n} = \vX_0 + \sigma_{t_n}\vZ$, where $\vX_0 \sim p_0$ and $\vZ \sim \mathcal N(\v0, I)$.
Consider the following objective:
\begin{gather}
        \E_{\vx_{t_n}}\E_{t \sim \mathcal U(t_n, T]}\E_{\vx_t = \vx_{t_n} + \sqrt{\sigma_t^2 - \sigma_{t_n}^2} \veta}\left[ \left|\left|\frac{\sigma_{t}^2 - \sigma_{t_n}^2}{\sigma_t^2} \vh_{\theta}(\vx_t, t)  + \frac{\sigma_{t_n}^2}{\sigma_t^2}\vx_t - \vx_{t_n}\right|\right|^2\right], \nonumber
        \label{eq:ours_obj}
\end{gather}

where $\veta$ in the above is a standard Gaussian vector. 
 Suppose that the family of functions $\{\vh_\theta\}$ is rich enough to contain the minimizer of the above objective overall functions $\vh(\vx,t)$. Then the minimizer $\theta^*$ of $J$ satisfies: 
\begin{gather}
    \vh_{\theta^*}(\vx_t, t) = \mathbb E[\vX_0 | \vX_t=\vx_t], \quad \forall \vx_t, t > t_n.
\end{gather}
\label{th:sure_alt}
\end{theorem}
\vspace{-1.5em}
The theorem above states that we can estimate the best $l_2^2$ denoisers for all noise levels $t: \sigma_t > \sigma_{t_n}$ without ever seeing clean data from $p_0$ and using an efficiently computable objective that contains no divergence term.
\paragraph{Proof Overview.} The central idea for this proof is to apply Tweedie's Formula twice, on appropriate random variables. We start by stating (a generalized version of) Tweedie's formula, the proof of which is given in the Appendix.
\begin{lemma}[Generalized Tweedie's Formula]
\label{lemma:Tweedies_general}
    Let:
    \begin{gather}
        \vX_t = \alpha_t \vX_0 + \sigma_t \vZ,
    \end{gather}
    for $\vX_0 \sim p_{0}$, $\vZ \sim \mathcal N(\v0, I)$, and some positive function $\alpha_t$ of $t$.
    Then,
    \begin{gather}
        \nabla_{\vx} \log p_t(\vx_t) = \frac{\alpha_t \E[\vX_0 | \vX_t = \vx_t] - \vx_t}{\sigma_t^2}.
    \end{gather}
\end{lemma}
For $t: \sigma_t > \sigma_{t_n}$, the R.V. $\vX_t$ can be written in the following two equivalent ways:
\begin{gather}
    \begin{cases}
        \vX_t = \vX_0 + \sigma_t \vZ \\
        \vX_t = \vX_{t_n} + \sqrt{\sigma_t^2 - \sigma_{t_n}^2}\vZ
    \end{cases}.
\end{gather}
By applying  Tweedie's formula twice, we get two alternative 
expressions for the same score-function since the distribution remains the same, irrespectively of how we choose to express $\vX_t$. By equating the two expressions for the score, we arrive at the following result:
\begin{gather}
    \E[\vX_{t_n} | \vX_t =\vx_t] = \frac{\sigma_t^2 - \sigma_{t_n}^2}{\sigma_t^2}\left(\E[\vX_0 | \vX_t = \vx_t] - \vx_t \right) + \vx_t. \nonumber
\end{gather}

We can train a network with denoising score matching to estimate $\E[\vX_{t_n} | \vX_t = \vx_t]$ and hence we can use the above equation to obtain $\E[\vX_0 | \vX_t=\vx_t]$, as desired.

The method we propose is conceptually similar to Noisier2Noise~\citep{moran2020noisier2noise} but instead of adding noise with a fixed magnitude to create further corrupted iterates, we consider a continuum of noise scales and we train the model jointly in a Denoising Score Matching fashion.

We underline that our method can be easily extended to the Variance Preserving (VP)~\citep{ncsnv3} case, i.e. when the available data are $\vX_{t_n} = \sqrt{1 - \sigma_{t_n}^2}\vX_0 + \sigma_{t_n}\vZ$. This is the setting for our Stable Diffusion finetuning experiments (see Section \ref{sec:experiments}). For the sake of simplicity, we avoid these calculations in the main paper and we point the interested reader to the Appendix (see Theorem \ref{th:sure_alt_vp}).

\paragraph{Developing Intuition.} A nice interpretation of our method is that it trains the network to predict the denoised image $\E[\vX_0 | \vX_t=\vx_t]$ by removing \textit{additional noise} that we introduced to the given samples $\vx_{t_n}$. This is similar to the idea of \textit{further corruption} developed in Ambient Diffusion~\citep{daras2023ambient}. The way we create further noisy samples $\vx_t$ given samples $\vx_{t_n}$ has some high-level connections to DDRM~\citep{ddrm} that reuses noise in the measurements to solve inverse problems with diffusion models.

\subsection{Learning the Optimal Denoiser for $\sigma_t \leq \sigma_{t_n}$}\label{subsec:consistency}
Theorem~\ref{th:sure_alt}  allows us to learn the optimal denoisers for $t: \sigma_t > \sigma_{t_n}$. However, to perform exact sampling we need to also learn $\E[\vX_0 | \vX_t =\vx_t]$ for $t: \sigma_t \le \sigma_{t_n}$. We achieve this by \textit{training the network to be consistent}.

\begin{definition}[Consistent Denoiser~\citep{daras2023consistent}]
Let $p_{\theta}(\vx_{t'}, t' | \vx_t, t)$ be the density of the sample $\vX_{t'}$ of the stochastic diffusion process of \Eqref{eq:reverse_theta} at time $t'$ when initialized with $\vx_t$ at time $t > t'$. The network $\vh_{\theta}(\cdot, t)$ that drives the process is a {\em consistent denoiser} if:
        \begin{gather}
            \vh_{\theta}(\vx_t, t) = \E_{\vX_{t'} \sim p_{\theta}(\vx_{t'}, t' | \vx_t, t)}\left[ \vh_{\theta}(\vX_{t'}, t')\right].
        \end{gather}
\label{def:consistency}
\end{definition}
\vspace{-1.5em}
The concept of consistency was introduced by~\citet{daras2023consistent} as a way to reduce error propagation in diffusion sampling and improve performance. Here, we find a completely different use case: we use consistency to learn the optimal denoisers for levels below the noise level of the available data. We are now ready to state our main theorem.
\begin{theorem}[Main Theorem (informal)]
    Define $\vX_t$ as in the beginning of Section~\ref{sec:background}. Suppose we are given samples $\vX_{t_n} = \vX_0 + \sigma_{t_n}\vZ$, where $\vX_0 \sim p_0$ and $\vZ \sim \mathcal N(\v0, I)$.
Consider the following objective:
    \begin{gather}
       \overbrace{\E_{\vx_{t_n}}\E_{t \sim \mathcal U[t_n, T]}\E_{\vx_t = \vx_{t_n} + \sqrt{\sigma_t^2 - \sigma_{t_n}^2} \veta}\left[ \left|\left|\frac{\sigma_{t}^2 - \sigma_{t_n}^2}{\sigma_t^2}\vh_{\theta}(\vx_t, t)+ \frac{\sigma_{t_n}^2}{\sigma_t^2}\vx_t - \vx_{t_n}\right|\right|^2\right]}^{\mathrm{Ambient \ Score \ Matching}}
        \nonumber \\ + \underbrace{\E_{t \sim \mathcal U(t_n, T], t' \sim \mathcal U(\epsilon, t), t'' \sim \mathcal U(t'-\epsilon, t')} \E_{\vx_{t}} \E_{\vx_{t'}|\vx_t}[
        \left|\left|\vh_{\theta}(\vx_{t'}, t') - \E_{\vx_{t''} \sim p_{\theta}(\vx_{t''}, t'' | \vx_{t'}, t')}\left[\vh_{\theta}(\vx_{t''}, t'')\right]\right|\right|^2\bigg]}_{\mathrm{Consistency \ Loss}},
        \label{eq:ours_obj_with_consistency_main_paper}
    \end{gather}
    where $\veta$ in the above is a standard Gaussian vector. 
 Suppose that the family of functions $\{\vh_\theta\}$ is rich enough to contain the minimizer  of the above objective overall functions $\vh(\vx,t)$. Then the minimizer $\theta^*$ satisfies: \begin{gather}
    \vh_{\theta^*}(\vx_t, t) = \mathbb E[\vX_0 | \vX_t=\vx_t], \quad \forall \vx_t,t.
\end{gather}
\label{th:main_theorem}
\end{theorem}
\vspace{-1em}
The formal statement and the proof of this Theorem is given in the Appendix (see Theorem \ref{th:main_theorem_formal_statement}).

\paragraph{Intuition and Proof Overview.} It is useful to build some intuition about how this objective works. There are two terms in the loss: i) the Ambient Score Matching term and ii) the Consistency Loss. The Ambient Score Matching term regards only noise levels $t: \sigma_t > \sigma_{t_n}$. Per Theorem \ref{th:sure_alt}, this term has a unique minimizer that is the optimal denoiser for all levels $t: \sigma_t > \sigma_{t_n}$. The consistency term in the loss, penalizes for violations of the Consistency Property (see Definition \ref{def:consistency}) for all pairs of times $t, t'$. The desired solution, $\vh(\vx_t,t)=\E[\vX_0 | \vX_t=\vx_t], \ \forall t,\vx_t$, minimizes the first term and makes the second term $0$, since it corresponds to a consistent denoiser. Hence, the desired solution is an optimal solution for the objective we wrote and the question becomes whether this solution is unique. The uniqueness of the solution arises from the Fokker-Planck PDE that describes the evolution of density: there is unique extension to a function that is $\E[\vX_0 | \vX_t=\vx_t], \ t: \sigma_t > \sigma_{t_n}$ and is consistent for all $t$. The latter result comes from Theorem 3.2 in Consistent Diffusion Models~\citep{daras2023consistent}.

\paragraph{Implementation Trade-offs and Design Choices.} When it comes to implementing the Consistency Loss there are trade-offs that need to be considered. First, we need to run partially the sampling chain. Doing so at every training step can lead to important slow-downs, as explained in~\citet{daras2023consistent}. To mitigate this, we choose the times $t', t''$ to be very close to one another, as in Consistent Diffusion, using a uniform distribution with support of width $\epsilon$. This helps us run only $1$ step of the sampling chain (without introducing big discretization errors) and it works because local consistency implies global consistency. Second, for the inner-term in the consistency loss we need to compute an expectation over samples of $p_{\theta}$. To avoid running the sampling chain many times during training, we opt for an unbiased estimator of this term that uses only two samples, following the implementation of~\citet{daras2023consistent}. Specifically, we use the approximation:
\begin{gather*}
    \left|\left|\vh_{\theta}(\vx_{t'}, t') - \E_{\vx_{t''} \sim p_{\theta}(\vx_{t''}, t'' | \vx_{t'}, t')}\left[\vh_{\theta}(\vx_{t''}, t')\right]\right|\right|^2 \nonumber \\ \approx 
    (\vh_{\theta}(\vx_{t''}^1, t'') -  \vh_{\theta}(\vx_{t'}, t'))^T (\vh_{\theta}(\vx_{t''}^2, t'') -  \vh_{\theta}(\vx_{t'}, t'')),
\end{gather*}
where $\vx_{t''}^1, \vx_{t''}^2$ are samples from $p_{\theta}( \cdot | \vx_t', t')$.
We finally note that our Consistency Loss defined in Eq. \ref{eq:ours_obj_with_consistency_main_paper} involves three expectations, instead of two, as in the original definition of \citet{daras2023consistent}. This is because the consistency property needs to hold for all pairs of times $(t', t'')$ and for $t' > t_n$ we don't have direct access to samples, i.e. we have to use the model to sample them given $\vx_t$.

\subsection{Testing Training Data Replication}
\label{sec:privacy_attack}
Learning from corrupted data is a potential mitigation strategy for the problem of training data replication. Thus, we need effective ways to evaluate the degree to which our models (and baselines trained on clean data) memorize. 

A standard approach is to generate a few thousand samples with the trained models (potentially using the dataset prompts) and then measure the similarities of the generated samples with their nearest neighbors in the dataset~\citep{somepalli2022diffusion, daras2023ambient}. This approach is known to ``systematically underestimate the amount of replication in Stable Diffusion and other models'', as noted by \citet{somepalli2022diffusion}.

We propose a novel attack that shows that diffusion models memorize their training sets at a higher rate than previously known. We use the trained diffusion priors to solve inverse problems at extremely high corruption levels and we show that the reconstructions are often almost perfect as long as the uncorrupted images belong to the training set.

\section{Experimental Evaluation}
\label{sec:experiments}

\subsection{Experiments with pre-trained models}
\label{sec:exp_privacy_attack}
In this section, we measure how much pre-trained foundation diffusion models memorize data from their training set. We perform our experiments with Stable Diffusion XL~\citep{podell2023sdxl} (SDXL), as it is the state-of-the-art open-source image generation diffusion model.

We take a random $10,000$ image subset of LAION and we corrupt it severely. We consider two models of corruption. In the first model, we take the LAION images and mask significant portions of them, as shown in Figure \ref{fig:sdxl_inp_attack}. The masked regions are selected automatically, using a YOLO object detection network~\citep{yolo}, to contain whole faces or large objects that are impossible to perfectly predict by only observing the non-masked content of the image. Yet, as seen in the last row of Figure \ref{fig:sdxl_inp_attack}, some posterior samples are almost pixel-perfect matches of the original images. This strongly indicates that the images in the top row of Figure \ref{fig:sdxl_inp_attack} were in the training set of SDXL and have been memorized.
Its important to note that the captions (from the LAION dataset) are entered as input in the inpainting model and this attack did not seem to work with null captions.

In the second corruption model, we encode LAION images with the SDXL encoder and we add a significant amount of noise to them. In Figure \ref{fig:sdxl_noise_attack}, we show images from LAION dataset, their encodings (visualizing them as $3$-channel RGB images), the noisy latents, the MMSE reconstruction (using the model's one-step prediction at the noise level of the corruption) and posterior samples from the model. Again, even if the corruption is severe and the MMSE denoised images are very blurry, the posterior samples from the model are very close to the original images from the dataset, indicating potential memorization.
In this corruption model the near-duplicate reconstructed images were obtained with null captions, so no text guidance was needed.

\begin{figure}[!htp]
    \centering
    \includegraphics[width=0.5\textwidth]{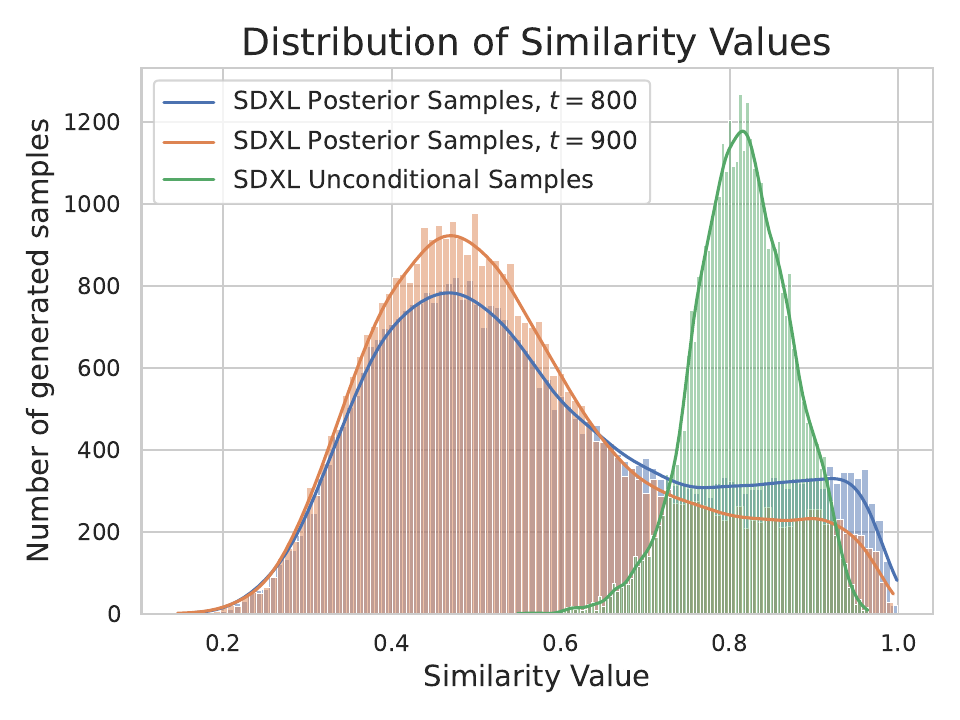}
    \caption{Distribution of image similarities of generated images with their nearest neighbors in the dataset for: i) the \citet{somepalli2022diffusion} method, and ii) for our noising method for two different noise levels. As shown, the fraction of images with similarities above 0.95 (near-identical to training set) is much higher for our method compared to the \citet{somepalli2022diffusion} baseline. 
    }
    \label{fig:sdxl_noise_memorization_curves}
\end{figure}

To quantify the degree of memorization and detect replication automatically, we adapt the methodology of \citet{somepalli2022diffusion}. In this work, the authors embed both the generated images and the dataset images to the DINOv2~\citep{oquab2023dinov2} latent space, and for each generated image compute its maximum inner product (similarity score) with its nearest neighbor in the dataset. We repeat this experiment, previously done for Stable Diffusion v1.4, for the latest SDXL model. We empirically find that similarities above $0.95$ correspond to almost identical samples to the ones in the training set and similarities above $0.9$ correspond to close matches.  We compare the distribution obtained using the \citet{somepalli2022diffusion} method with the distribution obtained using our noising approach (for two different noise levels) in Figure \ref{fig:sdxl_noise_memorization_curves}. As shown, our approach finds significantly more examples that have similarity values close to $1$. Also as mentioned, our attack did not need the prompts in this case. This is not necessarily surprising since our approach uses more information (the noisy latents) compared to the previously proposed method that only uses the prompts. 
Still, our results present evidence that diffusion models memorize significantly more training data compared to what was previously known.
For the inpainting case, we only compute embeddings for the infilled regions and hence the similarity numbers are not directly comparable. We present these results in Figure \ref{fig:sdxl_inp_memorization_curves} in the Appendix. 
\begin{figure}[h]
    \centering
    \includegraphics[width=0.5\textwidth]{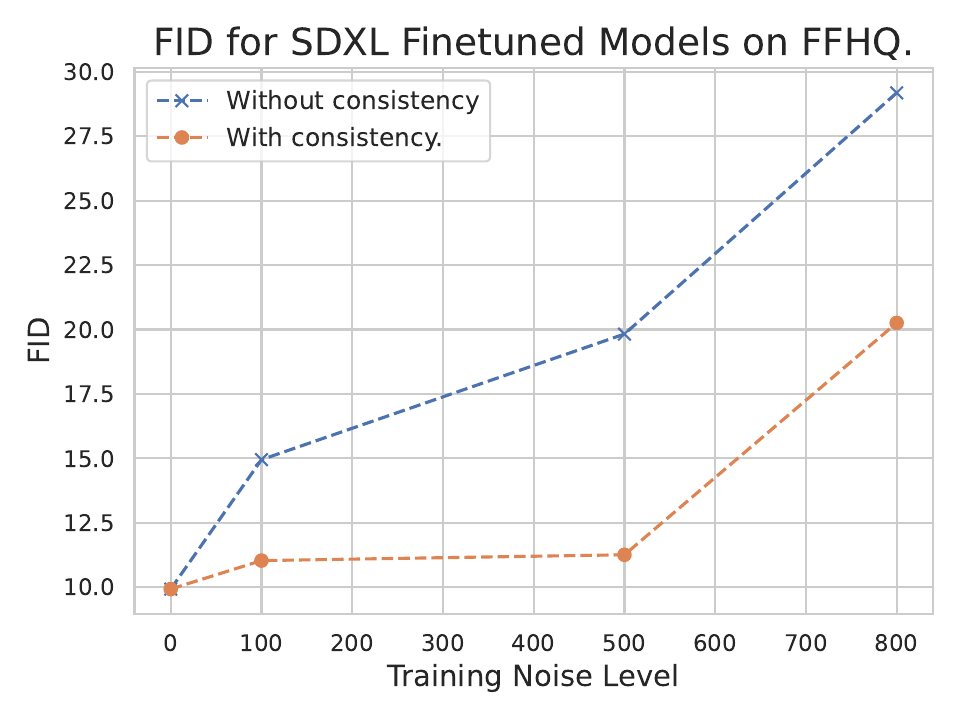}
    \caption{FID results for SDXL finetuned models, with and without consistency, on FFHQ, as we change the corruption level. The performance of models trained without consistency deteriorates significantly as we increase the corruption. Models trained with consistency maintain comparable performance to the baseline model (trained on clean data) for noise levels up to $t_n=500$.}
    \label{fig:fid_scores}
\end{figure}

\begin{table*}[!htp]
\centering
\begin{tabular}{c|c|c|c|c}
\hline
Noise Level Eval & Noise Level Training & Latent MSE & Pixel MSE \\
\hline
\multirow{3}{*}{900} & 100 & \textbf{0.3369} (0.0521) & 0.0802 (0.0257) \\
& 500 & 0.3377 (0.0514) & 0.0804 (0.0253) \\
& 800 &  0.3375 (0.0514) & \textbf{0.0796} (0.0254) \\
\hline
\multirow{3}{*}{800} & 100 & \textbf{0.2974} (0.0463) & \textbf{0.0566} (0.0219) \\
& 500 & 0.2978 (0.0464) & 0.0566 (0.0222) \\
& 800 &  0.3001 (0.0466) & 0.0570 (0.0220) \\
\hline
\multirow{3}{*}{500} & 100 & \textbf{0.2153} (0.0283) & \textbf{0.0219} (0.0092) \\
& 500 &  0.2159 (0.0283) & 0.0221 (0.0092) \\
& 800 &  0.2182 (0.0284) & 0.0226 (0.0094) \\
\hline
\multirow{3}{*}{100} & 100 & \textbf{0.0405} (0.0029) & \textbf{0.0068} (0.0027) \\
& 500 &  0.0409 (0.0029) & 0.0069 (0.0028) \\
& 800 &  0.0411 (0.0029) & 0.0070 (0.0028) \\
\hline
\end{tabular}
\caption{Restoration performance of models trained with noisy data at different noise levels. All the models have comparable performance, irrespective of the noise level of the dataset they were trained with.}
\label{tab:denoising_eval}
\end{table*}

\subsection{Finetuning Stable Diffusion XL}
\begin{figure}[!htp]
    \centering
    \begin{subfigure}{0.5\textwidth}
        \includegraphics[width=\textwidth]{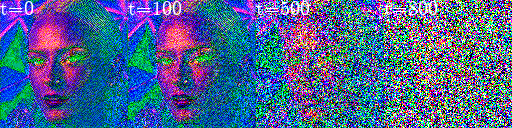}
        \caption{\small Noisy Latents for $t_n \in \{0, 100, 500, 800\}$.}
    \end{subfigure}
    \begin{subfigure}{0.5\textwidth}
        \includegraphics[width=\textwidth]{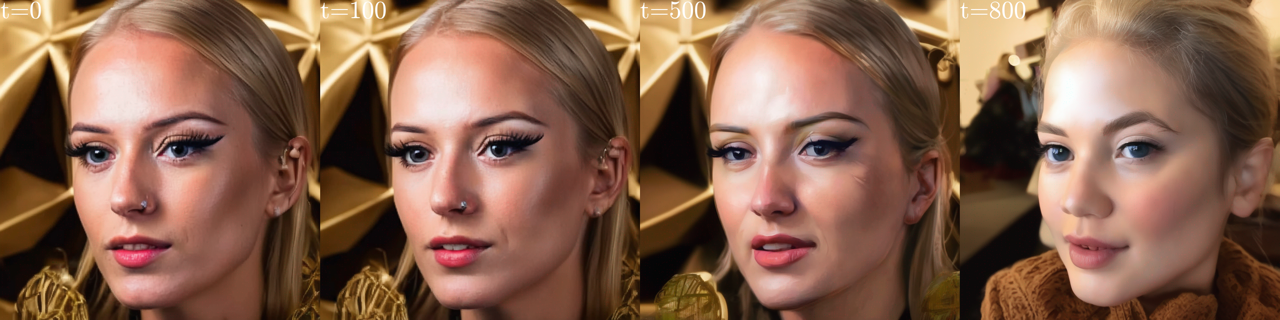}
        \caption{\small Posterior samples for $t_n \in \{0, 100, 500, 800\}$.}
    \end{subfigure}
    \caption{Visualization of the noise levels considered in the paper. The top row shows noisy latents, visualized as RGB images. The bottom row shows posterior samples obtained by the SDXL~\citep{podell2023sdxl} model given these noise latents.}
    \label{fig:noise_levels_vis}
\end{figure}

The next step is to use our framework, detailed in Sections \ref{subsec:double_tweedie}, \ref{subsec:consistency}, to finetune SDXL on corrupted data. We finetune our models on FFHQ, at $1024\times 1024$ resolution, since it is a standard benchmark for image generation. Given that SDXL is a latent model, we first encode the clean images using the SDXL encoder and then we add noise to the latents. We consider four noise levels which we will be referring to as: i) noiseless, $t_n=0$, $\sigma_{t_n}=0$, ii) low-noise, $t_n=100$, $\sigma_{t_n}=0.325$, iii) medium-noise, $t_n=500$, $\sigma_{t_n}=0.850$, and, iv) high-noise,  $t_n=800$, $\sigma_{t_n}=0.981$. For reference, we fix an image from the training set and we visualize posterior samples for each one of the noise levels in Figure \ref{fig:noise_levels_vis}. We train models with our Ambient Denoising Score Matching loss, with and without consistency. We provide the training details in the Appendix, Section \ref{sec:experimental_details}.

We first evaluate the denoising performance of our models. To do so, we take $32$ evaluation samples from FFHQ, we add noise to levels $t_{\mathrm{eval}} \in \{900, 800, 500, 100\}$, we use our trained models to denoise and we measure the reconstruction error. Since SDXL is a latent diffusion model, the noise (and the denoising) happens in the latent space. Hence, the MSE reconstruction error can be measured directly in the latent space or pixel space (by decoding the reconstructed latents). We present our results in Table \ref{tab:denoising_eval}. As shown, all the models have comparable performance across all noise levels, irrespective of the noise level of the data they saw during training. This is in line with our theory: all the models are trained to estimate $\E[\vx_0 | \vx_t]$ for all levels $t$.

\begin{figure*}[!ht]
  \centering
  \begin{subfigure}{0.45\textwidth}
    \includegraphics[width=\linewidth]{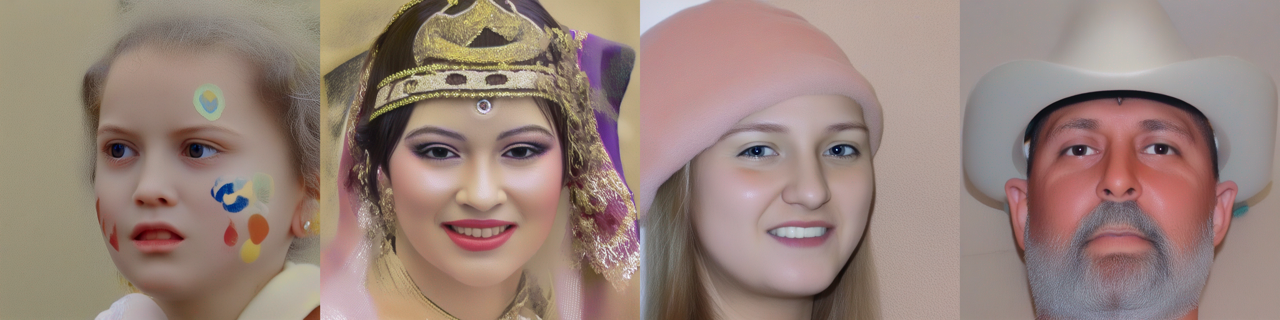}
    \caption{$t_n=0$.}
  \end{subfigure} \\
  \begin{subfigure}{0.45\textwidth}
    \includegraphics[width=\linewidth]{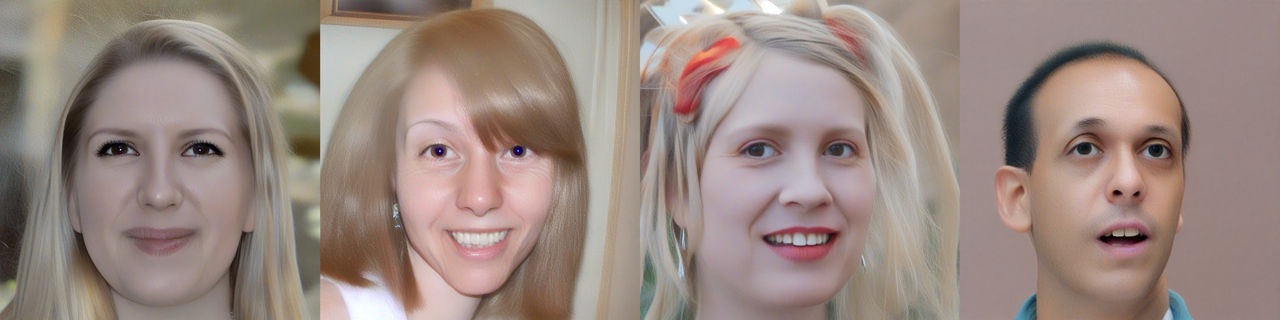}
    \caption{$t_n=100$, no consistency.}
  \end{subfigure}
  \begin{subfigure}{0.45\textwidth}
    \includegraphics[width=\linewidth]{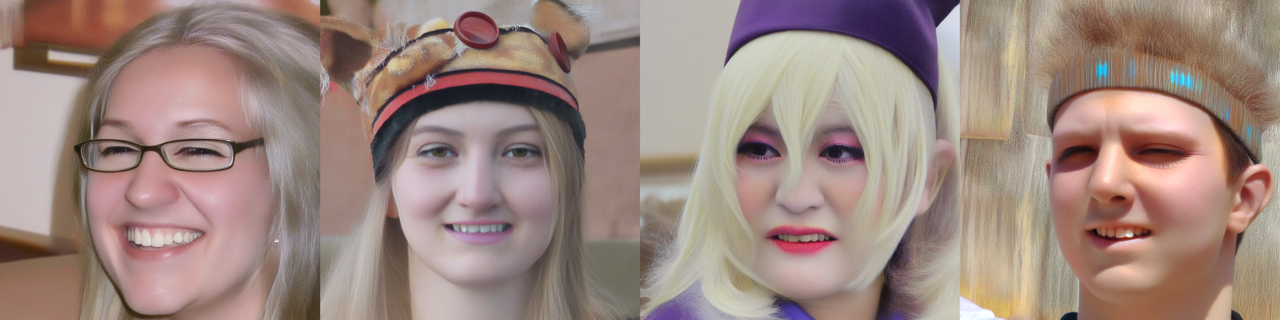}
    \caption{$t_n=100$, with consistency.}
  \end{subfigure} \\
  \begin{subfigure}{0.45\textwidth}
    \includegraphics[width=\linewidth]{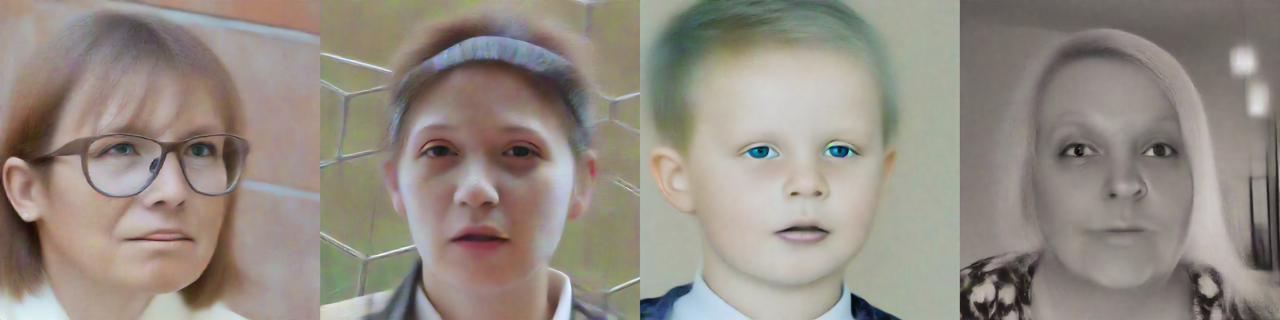}
    \caption{$t_n=500$, no consistency.}
  \end{subfigure} 
  \begin{subfigure}{0.45\textwidth}
    \includegraphics[width=\linewidth]{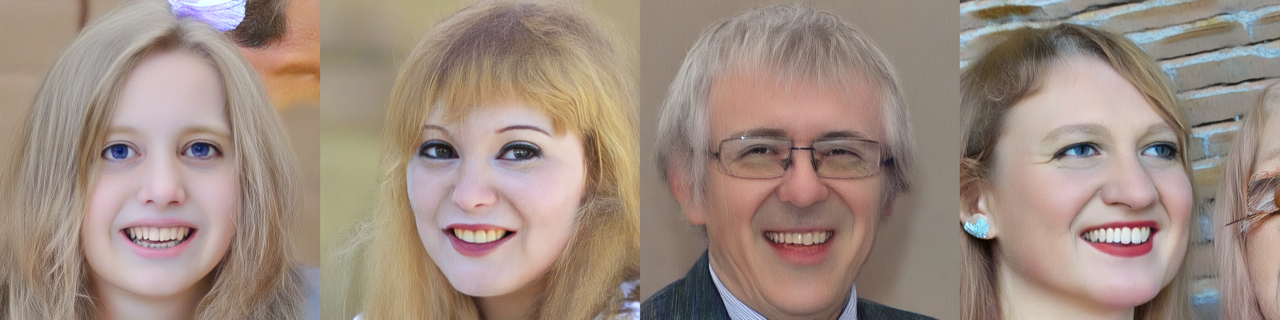}
    \caption{$t_n=500$, with consistency.}
  \end{subfigure}
  \\
  \begin{subfigure}{0.45\textwidth}
    \includegraphics[width=\linewidth]{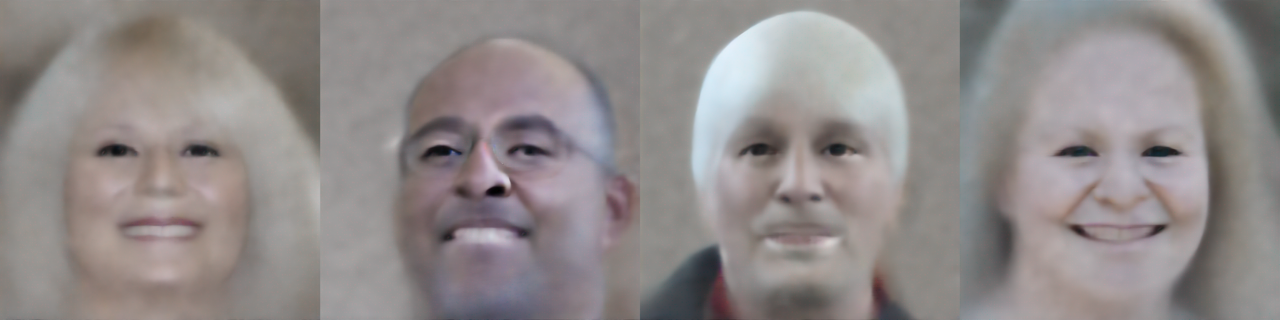}
    \caption{$t_n=800$, no consistency.}
  \end{subfigure}
  \begin{subfigure}{0.45\textwidth}
    \includegraphics[width=\linewidth]{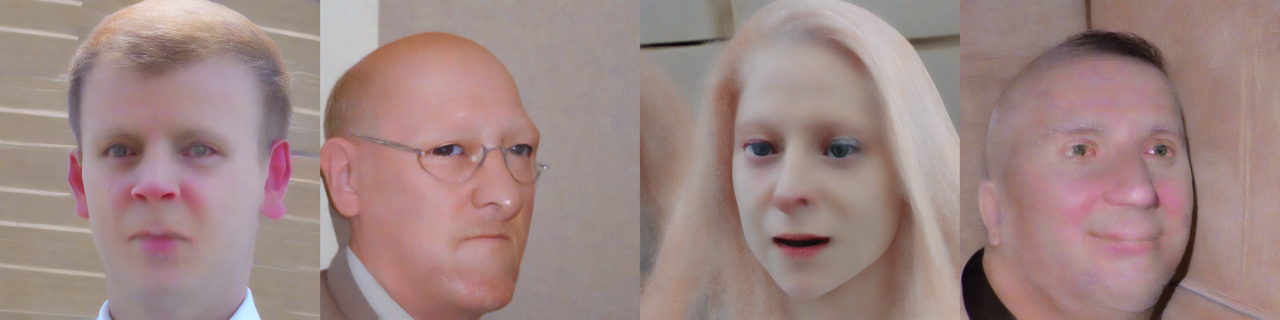}
    \caption{$t_n=800$, with consistency.}
  \end{subfigure}
  \caption{Unconditional generations for models trained with and without consistency at various noise levels $t_n$. Models trained without consistency lead to increasingly blurry generations as the noise level of the training data increases. Training with consistency recovers high-frequency details and leads to significantly improved images, especially for models trained on highly noisy data.}
  \label{fig:gens}
\end{figure*}

\begin{figure*}[!htp]
  \centering
  \begin{subfigure}{0.32\textwidth}
    \includegraphics[width=\linewidth]{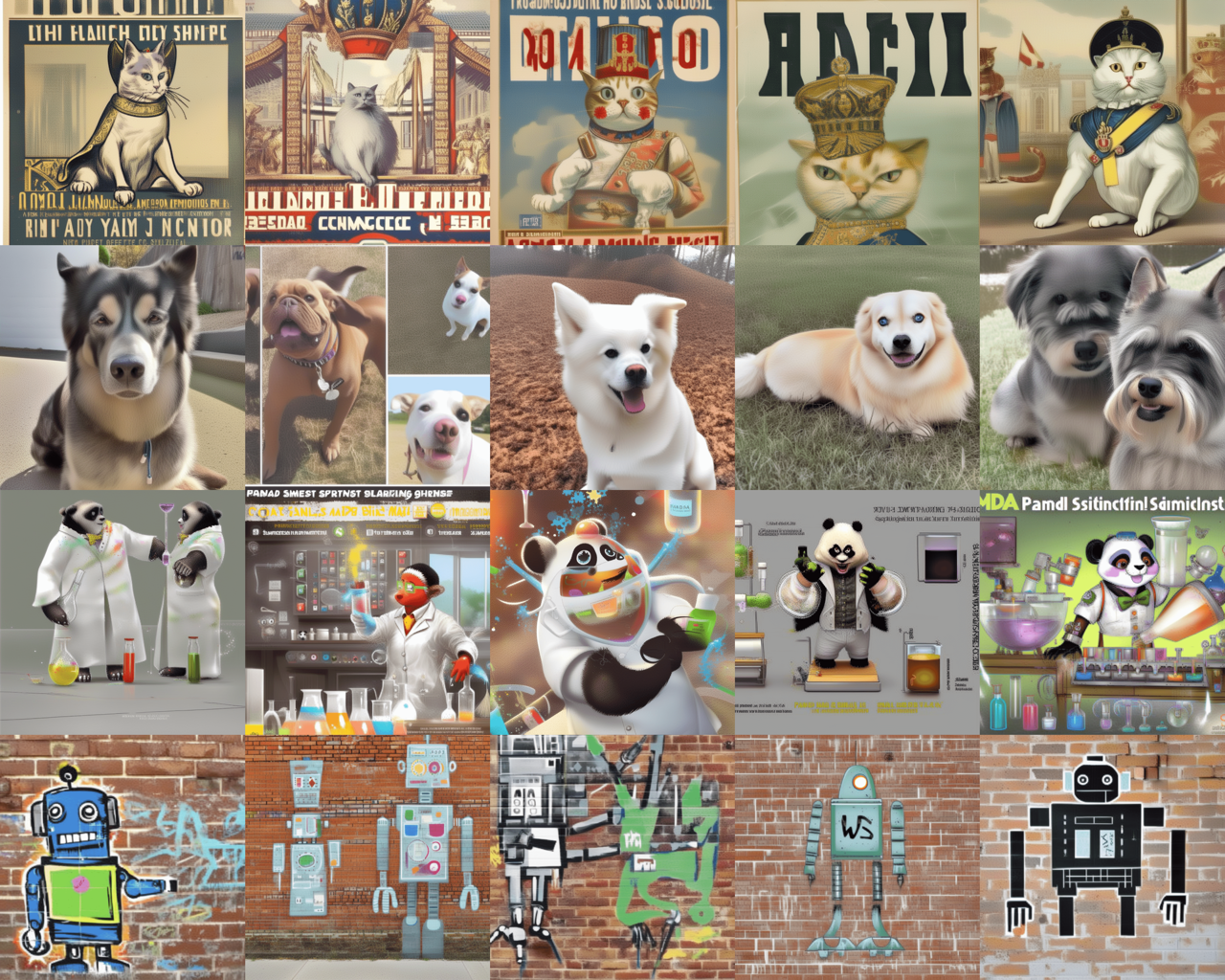}
    \caption{Finetuned model, no noise.}
  \end{subfigure}
  \begin{subfigure}{0.32\textwidth}
    \includegraphics[width=\linewidth]{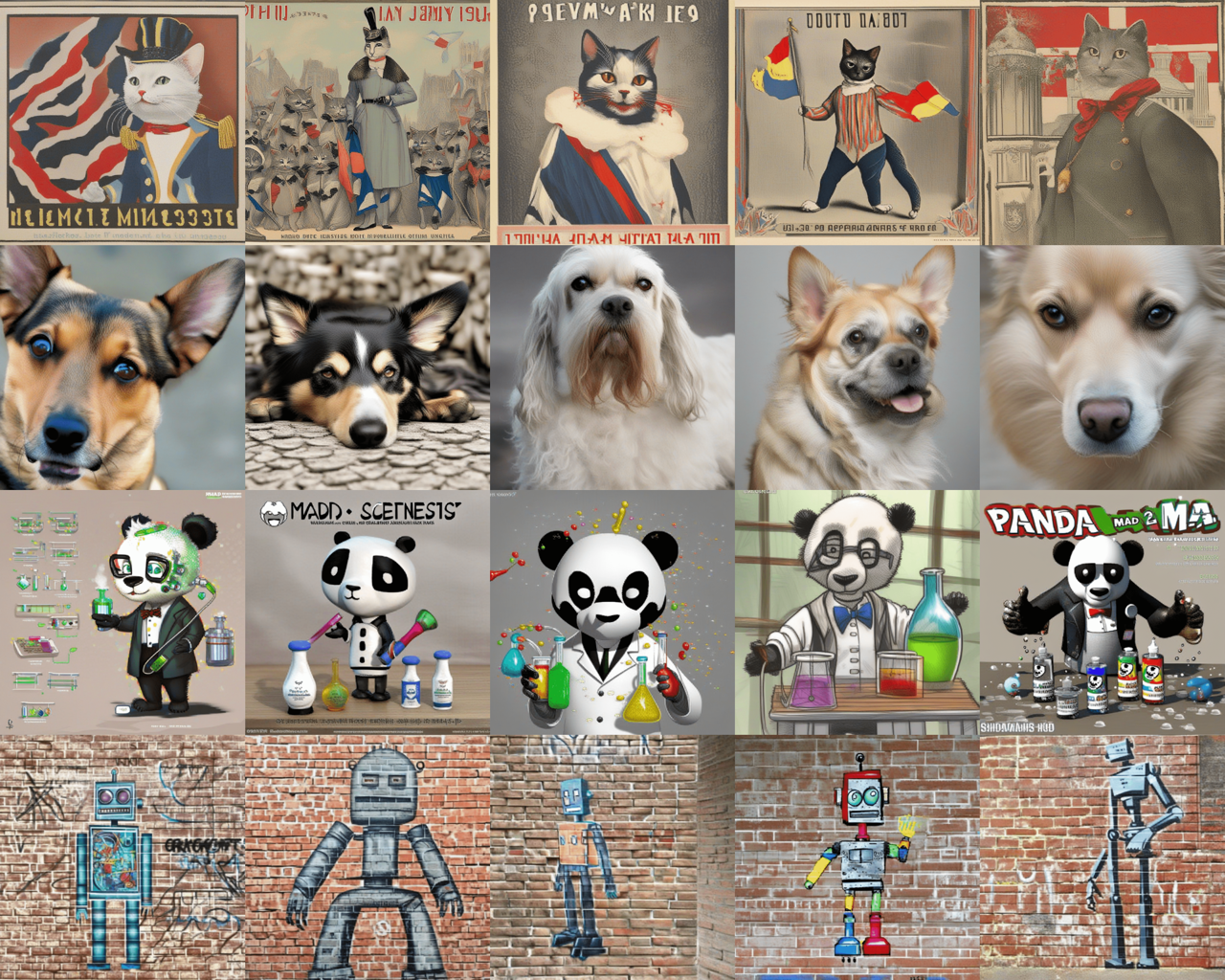}
    \caption{Finetuned model, $t_n=100$.}
  \end{subfigure}
  \begin{subfigure}{0.32\textwidth}  
    \includegraphics[width=\linewidth]{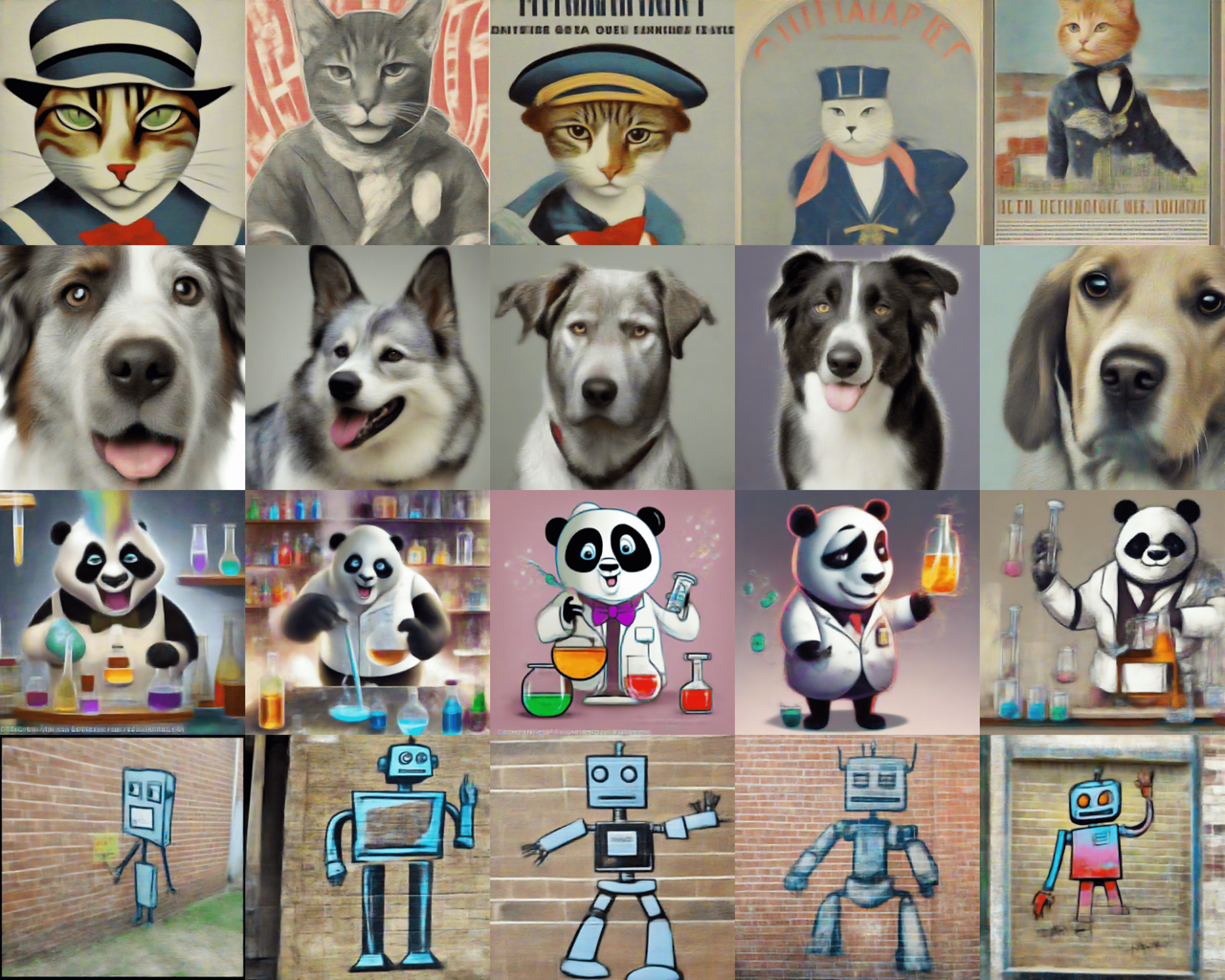}
    \caption{Finetuned model, $t_n=500$.}
  \end{subfigure}
  \caption{Generations of finetuned SDXL models on a 10k subset of LAION at different levels of noise in the training data. The following prompts were used: Row 1) ``A propaganda poster depicting a cat dressed as french emperor napoleon.'', Row 2) ``A high-quality image of a dog.'', Row 3) ``Panda mad scientist mixing sparkling chemicals, artstation.'', Row 4) ``A robot painted as graffiti on a brick wall.''.}
  \label{fig:finetuned_gens}
\end{figure*}

To understand better the role of consistency, we visualize unconditional samples from our models trained with and without consistency in Figure \ref{fig:gens}. As shown in the left column of Figure \ref{fig:gens}, models trained without consistency lead to increasingly blurry generations as the level of noise encountered during training increases. This is not surprising: as explained in Subsection \ref{subsec:prior_work_ambient_diffusion}, models trained without consistency sample from the distribution of MMSE denoised images, $\E[\vX_0 | \vX_{t_n}]$. As the noise level $t_n$ increases, these images become averaged and high-frequency detail is lost. As shown in the right column on Figure \ref{fig:gens}, training with consistency recovers high-frequency details and leads to significantly improved images, especially for models trained with highly noisy data ($t_n \in \{500, 800\}$).

We proceed to evaluate unconditional generation performance. For each of our models, we generate $50,000$ images and we compute the FID score. We visualize the performance of our models trained with and without consistency in Figure \ref{fig:fid_scores}. As shown, the performance of models trained without consistency deteriorates significantly as we increase the corruption. This is compatible with Figure \ref{fig:gens} (left) that shows that the generations become blurrier. Models trained on noisy data with consistency maintain comparable performance to the baseline model (trained on clean data) for noise levels up to $t_n=500$ and are better everywhere compared to their counterparts trained without consistency.

\subsection{Additional Finetuning Experiments}

We perform additional experiments to show that our framework can be used to fine-tune SDXL on datasets beyond FFHQ. We finetune SDXL on a 10k subset of LAION at different levels of training corruption and we show generations for different textual prompts at Figure \ref{fig:finetuned_gens}. As shown, even for high levels of training corruption, the model is capable of generating plausible images for arbitrary user prompts. 

To further show that our method can be used for data that follow a distribution significantly different to the training one, we finetune SDXL on a dataset of chest x-rays. In Figure \ref{fig:medical_gens}, we provide same samples of the training dataset (Row 1), generated samples without fine-tuning (Row 2), noisy samples that were used to fine-tune the model (Row 3), generated samples after fine-tuning without consistency (Row 4) and finally generated samples after fine-tuning with consistency. For all our generations, we use prompts from the dataset of interest. The generations of the model without fine-tuning are very different compared to the dataset samples, hinting that the model initially models a different distribution conditioned on the given prompt. After fine-tuning with noisy data, the generated samples are more closely related to the samples from the dataset. As we also observed in the rest of the experiments in this paper, consistency decreases the blurriness of the generated samples.

\subsection{Measuring Memorization of Finetuned Models}
\begin{figure}[!htp]
\centering
\includegraphics[width=0.5\textwidth]{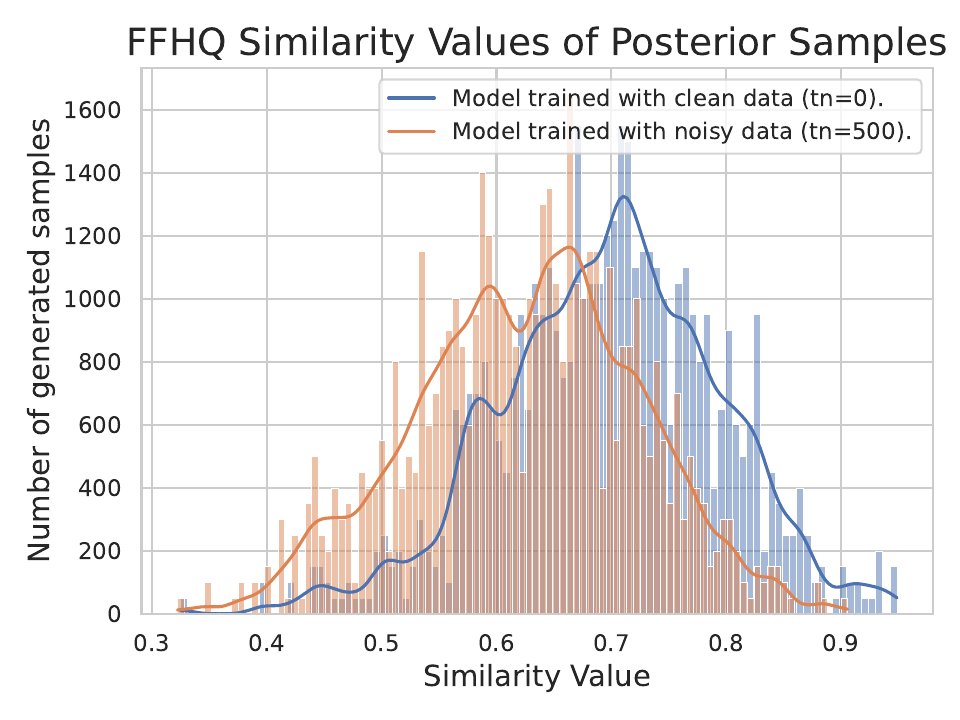}
    \caption{Distribution of similarities of posterior samples to their nearest neighbor in the dataset, given noisy latents (at $t=900$) for two models. The model trained with clean data (blue curve), has a distribution of similarity values that is more shifted to the right, indicating higher dataset memorization compared to the model trained with corrupted data (orange curve).}
    \label{fig:sdxl_noise_attack_finetuned}
\end{figure}
The final step in our experimental evaluation is to investigate to what extent training with noisy data reduced the rate of training data replication. To do so, we use the method we proposed in Section \ref{sec:exp_privacy_attack}. Specifically, we get the FFHQ training images, we encode them to the latent space of the SDXL Encoder and we add noise to them that corresponds to $t_n=900$. We then use the model trained with clean images and the model trained with data at $t_n=500$ noise level to perform posterior sampling, given the noisy latents. For each generated sample, we measure its DINO similarity to the nearest neighbor in the dataset. We plot the resulting distributions for the model trained with clean data and the $t_n=500$ model in Figure \ref{fig:sdxl_noise_attack_finetuned}. As shown, the model trained with clean data (blue curve), has a distribution of similarity values that is more shifted to the right, indicating higher dataset memorization compared to the model trained with corrupted data (orange curve). Finally, we once again compare with the method of~\citet{somepalli2022diffusion} for identifying training data replications. We use the model trained with clean data, we take the $50,000$ images that we used for FID generation and we compute their similarity to their nearest neighbor in the dataset. We compare with our approach in Figure \ref{fig:comp_our_vs_somepalli_finetuned} in the Appendix.

\section{Discussion and Other Related Work} The concurrent work of \citet{lu2024disguised} shows that an adversary can ``disguise'' copyrighted images in the training set. The implication is that training dataset inspection is not enough to detect whether copyrighted images have been used. This is finding conveys a similar message to our work since the training set might contain (severely corrupted) copyrighted images and pure inspection of the noisy images is not enough to determine if that's the case. Finally, we underline that the use of consistency enables sampling images below the observed data noise level, solving an open problem in the space of learning from corrupted data. For a more detailed exposition of diffusion models and consistency, we refer the interested reader to the relevant works of \citet{daras2023consistent, de2024target, boffi2023probability, shen2022self, albergo2023stochastic, lai2023equivalence, lai2023fp}.

\section{Conclusions, Limitations and Future Work}
\label{sec:conclusions}
We presented the first exact framework for training diffusion models to sample from an uncorrupted distribution using access to noisy data. We used our framework to finetune SDXL and we showed that training with corrupted data reduces memorization of the training set, while maintaining competitive performance. Our method has several limitations. First, it does not solve the problem of training diffusion models with \textit{linearly} corrupted data that provably sample from the uncorrupted distribution. Second, training with consistency increases the training time~\citep{daras2023consistent}. Finally, in some preliminary experiments on very limited datasets ($<100$ samples), the proposed Ambient Denoising Score Matching objective did not work. We plan to explore all these exciting open directions in future work.

\FloatBarrier
\pagebreak

\section*{Acknowledgements}
This research has been supported by NSF Grants AF 1901292, CNS 2148141, Tripods CCF 1934932, IFML CCF 2019844 and research gifts by Western Digital, Amazon, WNCG IAP, UT Austin Machine Learning Lab (MLL), Cisco and the Stanly P. Finch Centennial Professorship in Engineering. Constantinos Daskalakis has been supported by NSF Awards CCF-1901292, DMS-2022448 and DMS-2134108, a Simons Investigator Award, and the Simons Collaboration on the Theory of Algorithmic Fairness. Giannis Daras has been supported by the Onassis Fellowship (Scholarship ID: F ZS 012-1/2022-2023), the Bodossaki Fellowship and the Leventis Fellowship. Giannis Daras would also like to thank Yuval Dagan and Valentin De Bortoli for useful discussions.

\section*{Impact Statement}
This paper presents work whose goal is to advance the field of Machine Learning. There are many potential societal consequences of our work, none which we feel must be specifically highlighted here.

\bibliography{references}
\bibliographystyle{icml2024}

\newpage
\appendix
\onecolumn

\section{Theoretical Results}
In this section, we provide the proofs for the theoretical results of the main paper.
\subsection{Preliminaries}
We start by stating and proving a generalized version of Tweedie's formula that will be useful for learning the optimal denoisers, given only noisy data at $\sigma_{t_n}$, for noise levels higher than the level of the noise in the data, i.e. for $\sigma_t \geq \sigma_{t_n}$.

\begin{lemma}[Generalized Tweedie's Formula]
\label{lemma:Tweedies_general_appendix}
    Let:
    \begin{gather}
        \vX_t = \alpha_t \vX_0 + \sigma_t \vZ,
    \end{gather}
    for $\vX_0 \sim p_{0}$ and $\vZ \sim \mathcal N(\v0, I)$.
    Then,
    \begin{gather}
        \nabla_{\vx_t} \log p_t(\vx_t) = \frac{\alpha_t \E[\vX_0 | \vx_t] - \vx_t}{\sigma_t^2}.
    \end{gather}
\end{lemma}

\begin{proof}
    \begin{gather}
        \nabla_{\vx_t} \log p_t(\vx_t) = \frac{1}{p_t(\vx_t)}\nabla_{\vx_t} p_t(\vx_t) = \frac{1}{p_t(\vx_t)} \nabla_{\vx_t} \int p_t(\vx_t, \vx_0)\dx_0 \\
        = \frac{1}{p_t(\vx_t)} \nabla_{\vx_t} \int p_t(\vx_t | \vx_0) p_0(\vx_0)\dx_0 \\
        =  \frac{1}{p_t(\vx_t)}  \int \nabla_{\vx_t} p_t(\vx_t | \vx_0) p_0(\vx_0)\dx_0 \\
        = \frac{1}{p_t(\vx_t)}  \int p_t(\vx_t | \vx_0) \nabla_{\vx_t} \log p_t(\vx_t | \vx_0) p_0(\vx_0)\dx_0 \\
        =  \int p_0(\vx_0 | \vx_t) \frac{\alpha_t\vx_0 - \vx_t}{\sigma_t^2} \dx_0 \\
        = \frac{\alpha_t\E[\vX_0 | \vx_t] - \vx_t}{\sigma_t^2}.
    \end{gather}
\end{proof}

\subsection{Learning the Optimal Denoisers for $\sigma_t \geq \sigma_{t_n}$}

We will now use Lemma \ref{lemma:Tweedies_general} to connect the conditional expectation of $\vX_0$ given $\vX_{t}$, $\E[\vX_0 | \vX_t]$, to the conditional expectation of $\vX_{t_n}$ given $\vX_{t}$, $\E[\vX_{t_n}|\vX_t]$, for $t: \sigma_t \geq \sigma_{t_n}$. The latter can be learned with supervised learning and hence this connection will give us a way to learn how to find the best denoised image at level $t=0$ given only access to data at $t_{n}$.

\begin{lemma}[Connecting Conditional Expectations -- Variance Exploding]
    Let $\vX_{t_n} = \vX_0 + \sigma_{t_n}\vZ_1$ and $\vX_t = \vX_0 + \sigma_t \vZ_2, \quad \vZ_1, \vZ_2\sim \mathcal N(\v0, I)$ i.i.d. Then, for any $\sigma_t > \sigma_{t_n}$, we have that:
    \begin{gather}
        \E[\vX_0 | \vX_t = \vx_t] = \frac{\sigma_t^2}{\sigma_t^2 - \sigma_{t_n}^2}\E[\vX_{t_n} | \vX_t=\vx_t] - \frac{\sigma_{t_n}^2}{\sigma_t^2- \sigma_{t_n}^2}\vx_t.
    \end{gather}
    \label{lemma:cond_expectations_lemma_ve}
\end{lemma}

\begin{proof}
Applying Tweedie's formula (Lemma \ref{lemma:Tweedies_general}) for the pair $\vX_t, \vX_0$ we have:

\begin{gather}
        \nabla \log p_t(\vx_t) = \frac{\E[\vX_0 | \vx_t] - \vx_t}{\sigma_t^2}.
        \label{eq:tweedies_1}
\end{gather}
But also, $\vX_t = \vX_{t_n} + \sqrt{\sigma_t^2 - \sigma_{t_n}^2}\vZ$. Applying again Tweedie's formula for $\vX_t, \vX_{t_n}$ we get:
\begin{gather}
    \nabla \log p_t(\vx_t) = \frac{\E[\vX_{t_n} | \vx_t] - \vx_t}{\sigma_t^2 - \sigma_{t_n}^2}.
    \label{eq:tweedies_2}
\end{gather}

From \ref{eq:tweedies_1}, \ref{eq:tweedies_2}, we get:
    \begin{gather}
        \E[\vX_0 | \vX_t = \vx_t] = \frac{\sigma_t^2}{\sigma_t^2 - \sigma_{t_n}^2}\E[\vX_{t_n} | \vX_t=\vx_t] - \frac{\sigma_{t_n}^2}{\sigma_t^2- \sigma_{t_n}^2}\vx_t.
    \end{gather}
\end{proof}

We are now ready to present the proof of Theorem \ref{th:sure_alt}. For completeness, we restate the theorem here.
\begin{theorem}[Ambient Denoising Score Matching; restated Theorem~\ref{th:sure_alt}]
Define $\vX_t$ as in the beginning of Section~\ref{sec:background}. Suppose we are given samples $\vX_{t_n} = \vX_0 + \sigma_{t_n}\vZ$, where $\vX_0 \sim p_0$ and $\vZ \sim \mathcal N(\v0, I)$.
Consider the following objective:
    \begin{gather}
        J(\theta) = {\E_{\vx_{t_n}}\E_{t \sim \mathcal U(t_n, T]}\E_{\vx_t = \vx_{t_n} + \sqrt{\sigma_t^2 - \sigma_{t_n}^2} \veta}\left[ \left|\left|\frac{(\sigma_{t}^2 - \sigma_{t_n}^2)}{\sigma_t^2}\vh_{\theta}(\vx_t, t)+ \frac{\sigma_{t_n}^2}{\sigma_t^2}\vx_t - \vx_{t_n}\right|\right|^2\right]}, \nonumber
    \end{gather}
    where $\veta$ in the above is a standard Gaussian vector. 
 Suppose that the family of functions $\{\vh_\theta\}$ is rich enough to contain the minimizer  of the above objective overall functions $\vh(\vx,t)$. Then the minimizer $\theta^*$ of $J$ satisfies: 
\begin{gather}
    \vh_{\theta^*}(\vx_t, t) = \mathbb E[\vX_0 | \vX_t=\vx_t], \quad \forall \vx_t, t > t_n.
\end{gather}
\end{theorem}

\begin{proof}
    We start by using the definition of conditional expectation as the unique minimizer of the mean squared error objective. Specifically, we know that the solution to the optimization problem:
    \begin{gather}
        \tilde{J}(\theta) = \E_{\vx_{t_n}}\E_{t \sim \mathcal U(t_n, T]}\E_{\vx_t = \vx_{t_n} + \sqrt{\sigma_t^2 - \sigma_{t_n}^2} \veta}\left[ \left|\left|\vg_{\theta}(\vx_t, t) - \vx_{t_n}\right|\right|^2\right],
    \end{gather}
    for a rich enough family of functions $\vg_\theta$ is $\vg_{\theta^*}(\vx_t, t) = \E[\vX_{t_n} | \vX_t = \vx_t]$. We can parametrize $\vg_{\theta}$ as $\vg_{\theta}(\vx_t, t) = \frac{\sigma_t^2 - \sigma_{t_n}^2}{\sigma_{t}^2}\vh_{\theta}(\vx_t, t) + \frac{\sigma_{t_n}^2}{\sigma_t^2}\vx_t$ and solve the following optimization problem: 
    \begin{gather}
        \min_{\theta: \vg_{\theta}(\vx_t, t) = \frac{\sigma_t^2 - \sigma_{t_n}^2}{\sigma_{t}^2}\vh_{\theta}(\vx_t, t) + \frac{\sigma_{t_n}^2}{\sigma_t^2}\vx_t} \E_{\vx_{t_n}}\E_{t \sim \mathcal U(t_n, T]}\E_{\vx_t = \vx_{t_n} + \sqrt{\sigma_t^2 - \sigma_{t_n}^2} \veta}\left[ \left|\left|\vg_{\theta}(\vx_t, t) - \vx_{t_n}\right|\right|^2\right],
    \end{gather}
    which will have the same minimizer since any solution of $\tilde{J}(\theta)$ remains feasible. Hence,
    \begin{gather}
        \vg_{\theta^*}(\vx_t, t) = \E[\vX_{t_n} | \vX_t = \vx_t] \iff \\
        \frac{\sigma_t^2 - \sigma_{t_n}^2}{\sigma_{t}^2}\vh_{\theta^*}(\vx_t, t) + \frac{\sigma_{t_n}^2}{\sigma_t^2}\vx_t = \E[\vX_{t_n} | \vX_t = \vx_t].
    \end{gather}
    Using Lemma \ref{lemma:cond_expectations_lemma_ve}, the latter implies that $\vh_{\theta^*}(\vx_t, t) = \E[\vx_0 | \vX_t =\vx_t]$, as needed.
\end{proof}

\subsection{Extensions to Variance Preserving Diffusion}

The results that we presented can be easily extended to the Variance Preserving~\citep{ncsnv3} case, where the observed data are:
\begin{gather}
    \vX_{t_n} = \sqrt{1 - \sigma_{t_n}^2}\vX_0 + \sigma_{t_n}\vZ, \quad 0 < \sigma_{t_n} < 1.
\end{gather}

We first extend Lemma \ref{lemma:cond_expectations_lemma_ve}.

\begin{lemma}[Connecting Conditional Expectations -- Variance Preserving]
    Let $\vX_{t_n} = \sqrt{1 - \sigma_{t_n}^2}\vX_0 + \sigma_{t_n}\vZ_1, \quad 0 < \sigma_{t_n} < 1,$ and $\vX_t = \sqrt{1 - \sigma_t^2}\vX_0 + \sigma_t \vZ_2, \quad \vZ_1, \vZ_2\sim \mathcal N(\v0, I)$ i.i.d. Then, for any $1 > \sigma_t > \sigma_{t_n}$, we have that:
    \begin{gather}
        \E[\vx_0 | \vX_t=\vx_t] = \frac{\sigma_t^2}{\sigma_t^2 - \sigma_{t_n}^2}\sqrt{1 - \sigma_{t_n}^2}\E[\vx_{t_n} | \vX_t=\vx_t] - \sigma_{t_n}^2\frac{\sqrt{1 - \sigma_t^2}}{\sigma_t^2 - \sigma_{t_n}^2}\vx_t.
    \end{gather}
    \label{lemma:cond_expectations_lemma_vp}
\end{lemma}

\begin{proof}
Applying Tweedie's formula (Lemma \ref{lemma:Tweedies_general}) for the pair $\vX_t, \vX_0$ we have:

\begin{gather}
        \nabla \log p_t(\vx_t) = \frac{\sqrt{1 - \sigma_t^2}\E[\vX_0 | \vx_t] - \vx_t}{\sigma_t^2}.
        \label{eq:tweedies_1_vp}
\end{gather}    
The next step is to express $\vX_{t}$ as a function of $\vX_{t_n}$. We want to find co-efficients $\alpha, \beta$ such that:
\begin{gather}
    \vX_{t} = \alpha\vX_{t_n} + \beta \vZ_2 \iff \\
    \vX_t = \alpha \sqrt{1 - \sigma_{t_n}^2}\vX_0 + \alpha\sigma_{t_n} \vZ_1 + \beta \vZ_2 \iff \\
    \vX_t = \alpha \sqrt{1 - \sigma_{t_n}^2}\vX_0 + \sqrt{\alpha ^2 \sigma_{t_n}^2 + \beta^2} \vZ.
\end{gather}
Since $\vX_t = \sqrt{1 - \sigma_t^2}\vX_0 + \sigma_t \vZ$, this implies that the values of the desired co-efficients are:
\begin{gather}
    \alpha = \sqrt{\frac{1 - \sigma_t^2}{1 - \sigma_{t_n}^2}}, \quad \beta = \sqrt{\frac{\sigma_t^2 - \sigma_{t_n}^2}{1 - \sigma_{t_n}^2}},
\end{gather}
i.e. $\vX_t$ can be expressed as:
\begin{gather}
    \vX_t = \sqrt{\frac{1 - \sigma_t^2}{1 - \sigma_{t_n}^2}}\vX_{t_n} + \sqrt{\frac{\sigma_t^2 - \sigma_{t_n}^2}{1 - \sigma_{t_n}^2}} \vZ_2.
\end{gather}

By applying Tweedie's formula again for the $\vX_t, \vX_{t_n}$ pair, we get that:
\begin{gather}
    \nabla \log p_t(\vx_t) = \frac{\sqrt{\frac{1 - \sigma_t^2}{1 - \sigma_{t_n}^2}} \E[\vX_{t_n} | \vX_t = \vx_t] - \vx_t}{\frac{\sigma_t^2 - \sigma_{t_n}^2}{1 - \sigma_{t_n}^2}} \iff \\
    \nabla \log p_t(\vx_t) = \frac{\sqrt{(1 - \sigma_t^2) (1 -\sigma_{t_n}^2)}\E[\vX_{t_n} | \vX_t = \vx_t] - (1 - \sigma_{t_n}^2) \vx_t}{\sigma_t^2 - \sigma_{t_n}^2}.
    \label{eq:tweedies_2_vp}
\end{gather}
Finally, by equating \ref{eq:tweedies_1_vp}, \ref{eq:tweedies_2_vp}, we find that:
\begin{gather}
\E[\vx_0 | \vX_t=\vx_t] = \frac{\sigma_t^2}{\sigma_t^2 - \sigma_{t_n}^2}\sqrt{1 - \sigma_{t_n}^2}\E[\vx_{t_n} | \vX_t=\vx_t] - \sigma_{t_n}^2\frac{\sqrt{1 - \sigma_t^2}}{\sigma_t^2 - \sigma_{t_n}^2}\vx_t.
\end{gather}
\end{proof}

We can now use this result to write an objective for the VP case that learns $\E[\vX_0 | \vX_{t}], \forall \ t: \sigma_t > \sigma_{t_n}$, as in Theorem \ref{th:sure_alt}.

\begin{theorem}[Ambient Denoising Score Matching -- VP case]
Let $\vX_{t_n} = \sqrt{1 - \sigma_{t_n}^2}\vX_0 + \sigma_{t_n}\vZ, \quad \vZ \sim \mathcal N(\v0, I)$ and $\vX_t = \sqrt{1 - \sigma_t^2}\vX_0 + \sigma_t \vZ, \ t: 1 > \sigma_t > \sigma_{t_n} > 0$. Then, the unique minimizer of the objective:
    \begin{gather}
        J(\theta) = \E_{\vx_{t_n}}\E_{t \sim \mathcal U(t_n, T]}\E_{\vx_t | \vx_{t_n}}\left[\left|\left| \frac{\sigma_t^2 - \sigma_{t_n}^2}{\sigma_t^2 \sqrt{1 - \sigma_{t_n}^2}}\vh_{\theta}(\vx_t, t)  + \frac{\sigma_{t_n}^2}{\sigma_t^2}\sqrt{\frac{1 - \sigma_t^2}{1 - \sigma_{t_n}^2} \vx_t} - \vx_{t_n} \right|\right|^2\right]
    \end{gather}
is: \begin{gather}
    \vh_{\theta^*}(\vx_t, t) = \mathbb E[\vx_0 | \vX_t=\vx_t], \quad \forall t: 1 > \sigma_t > \sigma_{t_n}.
\end{gather}
\label{th:sure_alt_vp}
\end{theorem}

The proof of this Theorem is skipped for brevity since it follows the same steps as the proof of Theorem \ref{th:sure_alt}, with the only difference being that it invokes Lemma \ref{lemma:cond_expectations_lemma_vp} instead of Lemma \ref{lemma:cond_expectations_lemma_ve}.

\subsection{Learning the Optimal Denoisers for $\sigma_t \leq \sigma_{t_n}$}

We are now ready to present the theory for learning the optimal denoisers for $\sigma_t \leq \sigma_{t_n}$.  The formal version of our main Theorem (\ref{th:main_theorem}) is given below.

\begin{theorem}[Main Theorem]
    Define $\vX_t$ as in the beginning of Section~\ref{sec:background}. Suppose we are given samples $\vX_{t_n} = \vX_0 + \sigma_{t_n}\vZ$, where $\vX_0 \sim p_0$ and $\vZ \sim \mathcal N(\v0, I)$. Let also $p_{\theta}(\vx_{t'}, t' | \vx_t, t)$ be the density of the sample $\vX_{t'}$ sampled by the stochastic diffusion process of Equation~\ref{eq:reverse_theta},
    at time $t'$ when initialized with $\vx_t$ at time $t > t'$. Consider the following objective:

    \begin{equation}
        J(\theta) = \overbrace{\E_{\vx_{t_n}}\E_{t \sim \mathcal U(t_n, T]}\E_{\vx_t = \vx_{t_n} + \sqrt{\sigma_t^2 - \sigma_{t_n}^2} \veta}\left[ \left|\left|\frac{\sigma_{t}^2 - \sigma_{t_n}^2}{\sigma_t^2}\vh_{\theta}(\vx_t, t)+ \frac{\sigma_{t_n}^2}{\sigma_t^2}\vx_t - \vx_{t_n}\right|\right|^2\right]}^{\mathrm{Ambient \ Score \ Matching}}
        \nonumber \\  + \underbrace{\E_{t \sim \mathcal U(t_n, T], t' \sim \mathcal U(0, t), t''\sim \mathcal U[t' -\epsilon, t')}\E_{\vx_{t}} \E_{\vx_{t'} \sim p_{\theta}(\vx_{t'}, t' | \vx_t, t)}\left[\left|\left|\vh_{\theta}(\vx_{t'}, t') - \E_{\vx_{t''} \sim p_{\theta}(\vx_{t''}, t'' | \vx_{t'}, t')}\left[\vh_{\theta}(\vx_{t''}, t'')\right]\right|\right|^2\right]}_{\mathrm{Consistency \ Loss}},
        \label{eq:ours_obj_with_consistency}
    \end{equation}
where $\veta$ in the above is a standard Gaussian vector, and subject to:
\begin{gather}
    \begin{cases}
        \textrm{(A1):}\quad \vh_{\theta}(\vx_0, 0) = \vx_0, \forall \vx_0;\\
        \textrm{(A2):}\quad \frac{\vh_{\theta}(\vx_t, t) - \vx_t}{\sigma_t^2} =  \nabla \Phi(\vx_t), \quad \textrm{for some scalar-valued function $\Phi$, $\forall t, \vx_t$}.
    \end{cases}
\end{gather}
 
 Suppose that the family of functions $\{\vh_\theta\}$ is rich enough to contain the minimizer  of the above objective overall functions $\vh(\vx,t)$. Then the  minimizer $\theta^*$ of $J$ satisfies:
\begin{gather}
    \vh_{\theta^*}(\vx_t, t) = \mathbb E[\vX_0 | \vX_t=\vx_t], \quad \forall \vx_t,t.
\end{gather}
\label{th:main_theorem_formal_statement}
\end{theorem}

\begin{proof}
    The first term of the loss involves predictions of the network only for $t: \sigma_t > \sigma_{t_n}$.
    By Theorem \ref{th:sure_alt}, for these times $t$, there is a unique minimizer and the solution should satisfy:
    \begin{gather}
        \vh_{\theta^*}(\vx_t, t) = \E[\vX_0 | \vX_t=\vx_t], \quad \forall t: \sigma_t > \sigma_{t_n}.
        \label{eq:first_term_min}
    \end{gather}
    The solution:
    \begin{gather}
        \vh_{\theta^*}(\vx_t, t) = \E[\vX_0 | \vX_t=\vx_t], \quad \forall t,
    \end{gather}
    is one minimizer of the loss since: i) it is a feasible solution (satisfies (A1), (A2)), ii) it satisfies the condition of \Eqref{eq:first_term_min} that corresponds to the minimization of the first term, and iii) it makes the second term of the loss $0$ (by the tower law of expectation). Hence, the only thing left to show is that the solution is unique for times $t: \sigma_t \leq \sigma_{t_n}$.

    Let $\vh_{\tilde \theta}$ be another optimal solution. It has to satisfy the following properties:
    \begin{enumerate}
        \item $\vh_{\tilde \theta}$ needs to make the second term in the loss $0$, i.e. $\vh_{\tilde \theta}$ is a consistent denoiser (see Definition \ref{def:consistency}) for all $t$. This is because we found another minimizer that minimizes the first term of the loss and makes the second term $0$. 
        \item $\vh_{\tilde \theta}$ satisfies (A1), (A2) since the optimal solution should be a feasible solution.
        \item  $\vh_{\tilde \theta}$ needs to satisfy Eq. \ref{eq:first_term_min}, since the first term in the loss has a unique minimizer. 
    \end{enumerate}  
    
    By Theorem 3.2 (part ii) of Consistent Diffusion Models~\citep{daras2023consistent}, the only function that satisfies properties 1., 2., 3. is the function $\vh_{\theta^*}$ and hence the solution is unique.

\end{proof}

\section{Experimental Details}
\label{sec:experimental_details}

In this section, we provide further regarding the SDXL finetuning experiments. We train all our models with a batch size of $16$ using a constant learning rate $1e-5$. For all our experiments, we use the Adam optimizer with the following hyperparameters: $\beta_1=0.9, \beta_2=0.999, \mathrm{weight \ decay}=0.01$. We train all of our models for at least $200,000$ steps or roughly $45$ epochs on FFHQ. The models trained with consistency were finetuned for $50,000$ steps, initialized from the models trained without consistency after $150,000$ steps. We did this to save computation time since training with consistency loss takes $\approx 3\times$ more time compared to vanilla training. During finetuning, we used a weight of $0.01$ for the consistency loss for the $t_n \in \{100,500\}$ models and a weight of $1e-4$ for our $t_n=800$ model. We noticed that further increasing the weight for the latter led to training collapse.

For our finetuning, we use LoRA with rank $4$, following the implementation of SDXL finetuning from the \texttt{diffusers} Github repository. We train all of our models on 16-bit precision to reduce the memory requirements and accelerate training speed. For all the experiments in the paper (including FID evaluation) the images were generated using $25$ inference steps and the DDIM~\citep{ddim} sampling algorithm. We underline that better performance could have been achieved by increasing the number of steps, the training and sampling precision and by carefully tuning the batch size. However, in this paper, we did not optimize for state-of-the-art unconditional generation performance but rather we focused on building a complete and exact framework for learning diffusion models from noisy data.

\section{Additional Results}
In this section, we provide additional results that were not included in the main paper. Figure \ref{fig:sdxl_inp_memorization_curves} shows the memorization curves for the SDXL inpainting experiment (see also Figure \ref{fig:sdxl_inp_attack}). Figure \ref{fig:comp_our_vs_somepalli_finetuned} compares the~\citet{somepalli2022diffusion} method for detecting training data replication with our proposed method that works by denoising extremely corrupted encodings of dataset images. We once again underline that it is not surprising that our method indicates higher memorization since it has access to more information (the noisy latents).

\begin{figure}[!htp]
    \centering
    \includegraphics[width=0.5\textwidth]{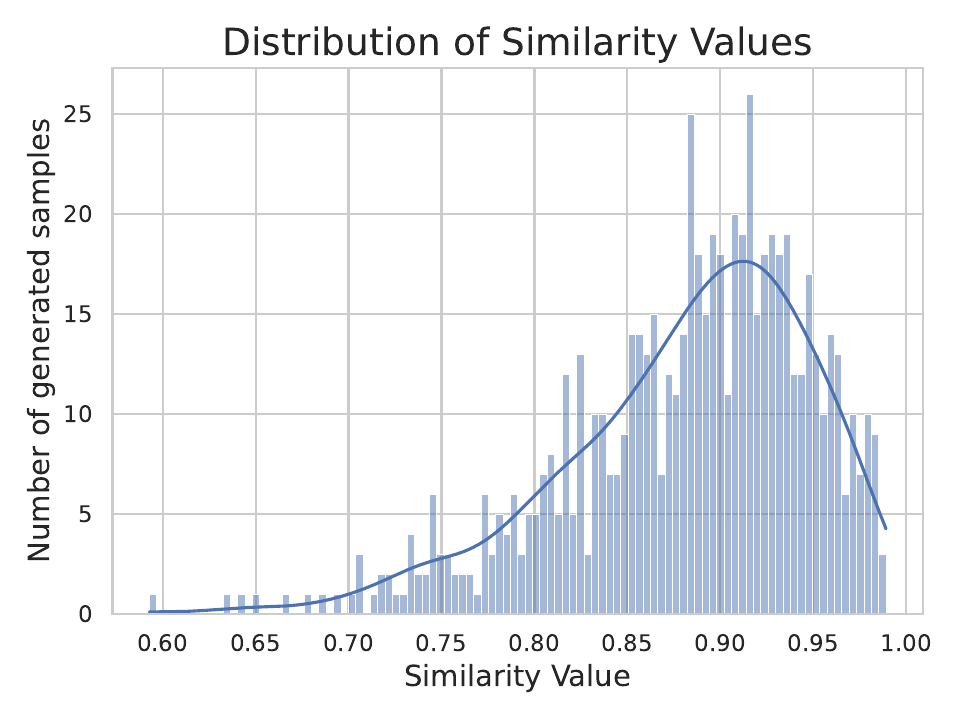}
    \caption{Distribution of image similarities of generated images with their nearest neighbors in the dataset for inpainting attack.}
    \label{fig:sdxl_inp_memorization_curves}
\end{figure}

\begin{figure}[!htp]
    \centering
    \includegraphics[width=0.5\textwidth]{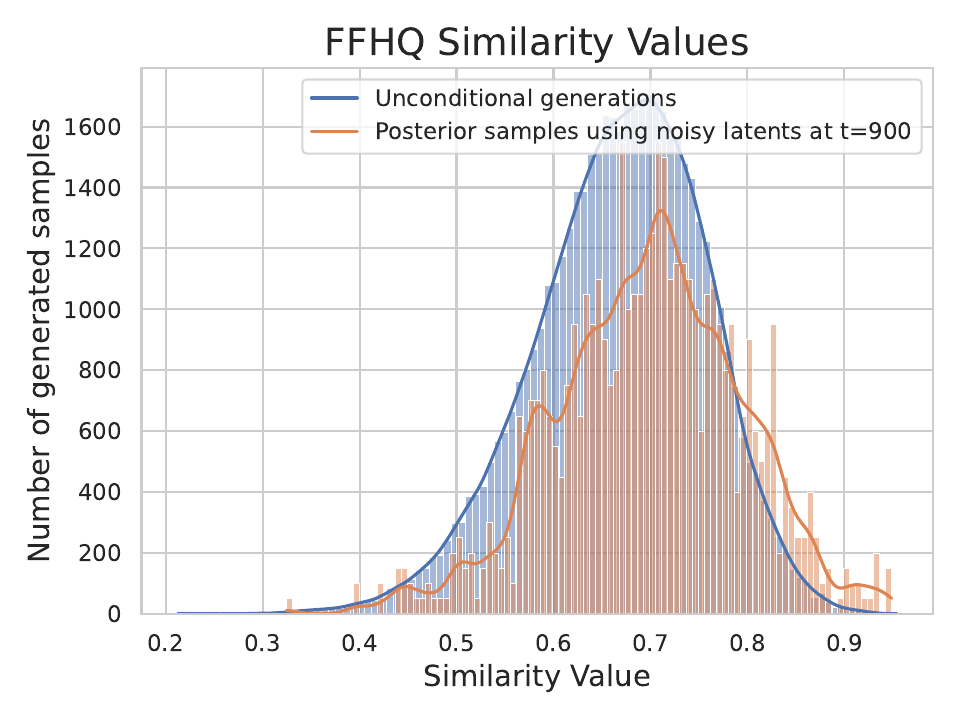}
    \caption{Comparison of \citep{somepalli2022diffusion} method for detecting training data replication with our proposed method that works by denoising extremely noisy dataset latents. The comparison is given for an SDXL model finetuned on clean FFHQ images. Our method (orange curve) gives a distribution that is more shifted to the right, indicating higher dataset memorization. This is not surprising since we have access to more information (noisy latents) compared to the baseline method.}
    \label{fig:comp_our_vs_somepalli_finetuned}
\end{figure}

\begin{figure*}[!ht]
  \centering
  \begin{subfigure}{0.7\textwidth}
    \includegraphics[width=\linewidth]{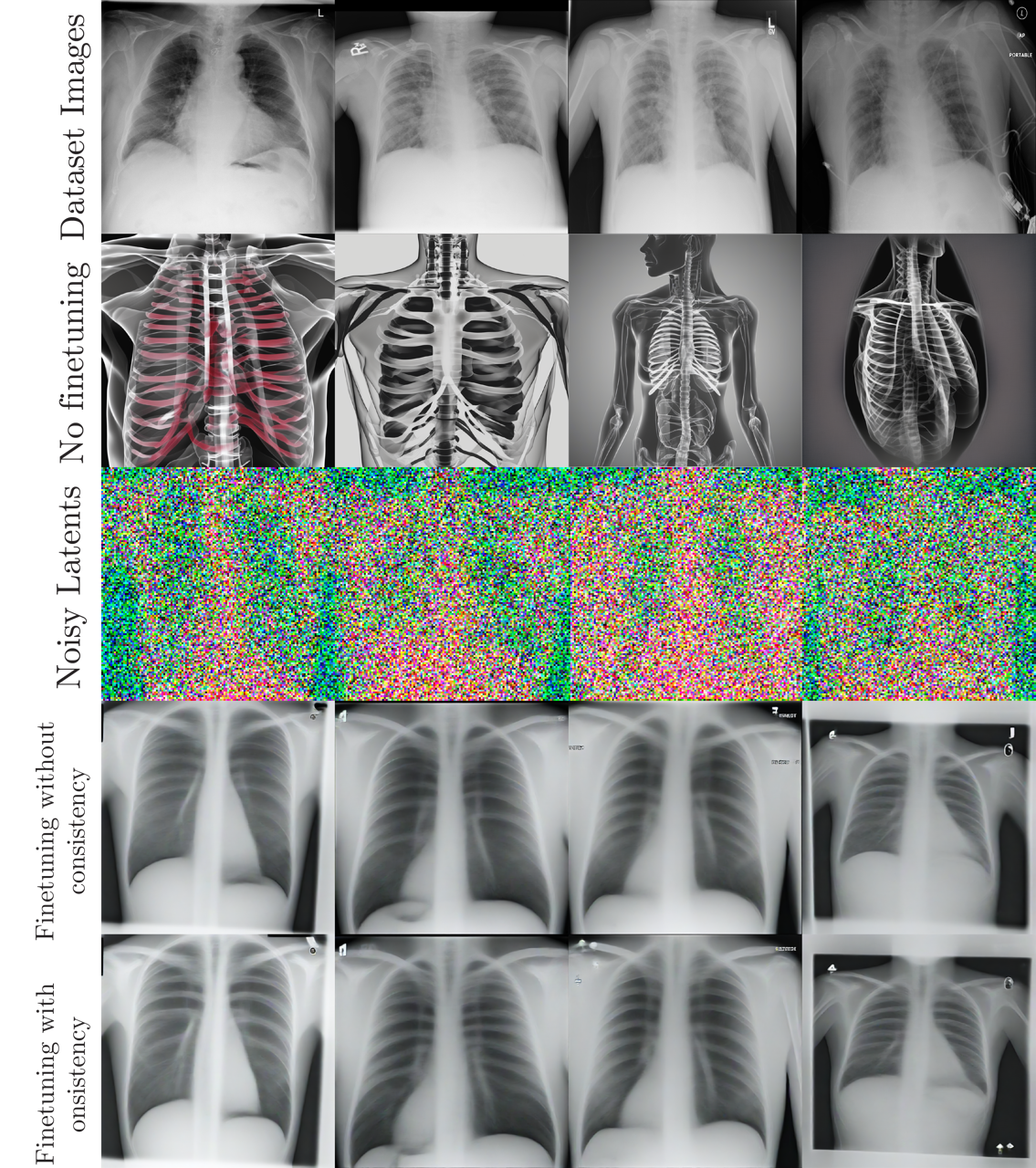}
  \end{subfigure}
  \caption{SDXL Finetuning on a dataset of X-rays. Row 1: samples of the training dataset, Row 2: generated samples without fine-tuning, Row 3: noisy samples that were used to fine-tune the model, Row 4: generated samples after fine-tuning without consistency, and, Row 5: generated samples after fine-tuning with consistency.}
  \label{fig:medical_gens}
\end{figure*}

\end{document}